\documentclass[a4paper]{scrartcl}

\usepackage{doc}

\usepackage{amsthm}
\usepackage{amssymb}
\usepackage{amsmath}
\usepackage{thm-restate}
\usepackage{enumitem}
\usepackage{array}
\usepackage{tikz}
\usetikzlibrary{arrows} 
\usetikzlibrary{positioning}
\usetikzlibrary{fit}
\usetikzlibrary{decorations.pathmorphing}
\usetikzlibrary{patterns,patterns.meta,decorations.pathreplacing}
\tikzset{
    edge/.style={->,> = latex'}
}

\newtheorem{theorem}{Theorem}
\newtheorem{example}[theorem]{Example}

\usepackage{hyperref}
\usepackage[noabbrev,capitalize]{cleveref}
\crefname{enumi}{property}{properties}
\usepackage{todonotes}

\usepackage{ifthen}
\newboolean{proofsInPaper}
\setboolean{proofsInPaper}{true}

\usepackage{algorithm}
\usepackage[noend]{algpseudocode}

\makeatletter
\newenvironment{breakablealgorithm}
  {
   \begin{center}
     \refstepcounter{algorithm}
     \hrule height.8pt depth0pt \kern2pt
     \renewcommand{\caption}[2][\relax]{
       {\raggedright\textbf{\fname@algorithm~\thealgorithm} ##2\par}%
       \ifx\relax##1\relax 
         \addcontentsline{loa}{algorithm}{\protect\numberline{\thealgorithm}##2}%
       \else 
         \addcontentsline{loa}{algorithm}{\protect\numberline{\thealgorithm}##1}%
       \fi
       \kern2pt\hrule\kern2pt
     }
  }{
     \kern2pt\hrule\relax
   \end{center}
  }
\makeatother

\newcommand{\eps}{\varepsilon}
\newcommand{\theory}{$\mathcal{T}$}

\newcommand{\shortrule}[6]{\vspace{1.5ex}\noindent\begin{minipage}{#6ex}{\bfseries #2}\end{minipage} $\qquad$ #3 $\quad\Rightarrow_{\text{#1}}\quad$ #4 \par\smallskip\noindent #5\vspace{1ex}}
\newcommand{\longerrule}[6]{\begin{samepage}\vspace{1.5ex}\noindent\begin{minipage}{#6ex}{\bfseries #2}\end{minipage} $\qquad$ #3 \begin{flushright}\vspace{-1.5ex}$\Rightarrow_{\text{#1}}\quad$ #4\end{flushright}\end{samepage}\par\vspace{-1.5ex}\smallskip\noindent #5\vspace{1ex}}
\newcommand{\ruleapp}[2]{\Rightarrow_{\text{#1}~~~~~~~~~~~~~}^{\text{#2}}\quad}
\newcommand{\shortruleapp}[2]{\Rightarrow_{\text{#1}}^{\text{#2}}}
\DeclareMathOperator{\atom}{atom}
\DeclareMathOperator{\comp}{comp}
\DeclareMathOperator{\cb}{cb}
\DeclareMathOperator{\clauses}{clauses}
\DeclareMathOperator{\resolve}{resolve}
\DeclareMathOperator{\mgu}{mgu}
\DeclareMathOperator{\gnd}{gnd}
\DeclareMathOperator{\reach}{reach}
\DeclareMathOperator{\dom}{dom}
\DeclareMathOperator{\codom}{codom}

\DeclareMathOperator{\numT}{numT}
\DeclareMathOperator{\multiGround}{multiGround}

\newcommand{\ICLF}{ICLF}
\newcommand{\CDCLW}{CDCL\_\hspace{-1.5pt}W}

\newcommand{\inst}[1]{\textsuperscript{#1}}
\newcommand{\orcidID}[1]{}
\newcommand{\authorrunning}[1]{}
\newcommand{\titlerunning}[1]{}
\newcommand{\email}[1]{\texttt{#1}}
\newcommand{\institute}[1]{\date{\normalsize{\def\and{\\[0.3em]}#1}}}

\title{Extending SMT Solving with Non-Ground Clause Learning}
\titlerunning{Extending SMT Solving with Non-Ground Clause Learning}

\author{Yasmine Briefs\inst{1,2}\orcidID{0009-0007-9917-7517}
\and Christoph Weidenbach\inst{1}\orcidID{0000-0001-6002-0458}
}
\authorrunning{Y.~Briefs and C.~Weidenbach}

\institute{Max Planck Institute for Informatics, Saarbrücken, Germany\\
\email{\{ybriefs,weidenbach\}@mpi-inf.mpg.de} \and
Graduate School of Computer Science,\\Saarland Informatics Campus, Saarbrücken, Germany}

\begin{document}

\maketitle

\begin{abstract}
Quantifier instantiation is currently the main approach to non-ground SMT solving: solvers
generate ground instances and solve the resulting ground SMT problems with
CDCL(T)-style reasoning. When a conflict is found, conflict analysis learns only a
ground clause, even though the conflict comes from instances of non-ground
clauses. Yet non-ground reasoning can give exponentially shorter
proofs than purely ground reasoning. We propose a calculus
that consists of ground instantiations, CDCL(T)-style rules, and non-ground
conflict analysis. The solver reasons on ground instances, but the resolution
steps of conflict analysis are performed on their original non-ground
clauses. This produces learned clauses that are typically more general than the
ground conflict. With a suitable strategy, the learned clauses are even
non-redundant. We also show how chronological backtracking can be included in
SMT solving. Our calculus gives a common setting for CDCL(T)-style SMT
solving, a range of instantiation-based procedures, and non-ground clause
learning, and we prove that it simulates CDCL, SCL(FOL), SCL(T), and even Resolution.
\end{abstract}

\section{Introduction}\label{section:introduction}
Satisfiability Modulo Theories (SMT) has become a central tool in automated reasoning, both in its ground form and in the presence of quantifiers.
At the ground level, much of this success is due to CDCL(T)~\cite{DBLP:journals/jacm/NieuwenhuisOT06,DBLP:series/faia/BarrettSST21}, which combines conflict-driven clause learning (CDCL)~\cite{DBLP:series/faia/0001LM21} with dedicated theory reasoning.
To handle non-ground problems, this architecture is commonly extended by quantifier instantiation~\cite{DBLP:journals/jacm/DetlefsNS05,DBLP:conf/cade/GeBT07,DBLP:conf/cav/GeM09,DBLP:conf/cade/Reynolds16,DBLP:series/natosec/Bjorner16}: quantified formulas are instantiated, and the resulting ground problems are solved by CDCL(T).
Treating solver components, such as CDCL(T), as black boxes is a powerful design principle in automated reasoning.
At the same time, many advances in solver efficiency come from integrating these components more tightly with the surrounding search procedure.
This observation also motivates a closer integration of quantifier instantiation with CDCL(T).
In current SMT solving, quantified clauses are instantiated and the resulting ground conflicts are analyzed by CDCL(T), yielding learned ground clauses.
Our goal is to lift this learning step to the non-ground clauses from which the conflicting instances originate.
The motivation is not only conceptual: non-ground learning can yield exponentially shorter refutations than purely ground learning~\cite{DBLP:conf/cade/PerezV08}.

To formalize this idea, we introduce the calculus \ICLF{} (Instantiation-based Clause Learning Framework).
The input to \ICLF{} is a first-order clause set modulo theories.
A state consists of three main components: a set $N$ of possibly non-ground clauses, initially containing the input clauses; a set $G$ of ground instances of clauses in $N$, initially empty; and a ground partial model $\Gamma$, called the \emph{trail}, also initially empty.
The rules of \ICLF{} fall into four groups.
Instantiation rules create and delete ground instances of clauses in $N$.
Ground rules build a partial model for $G$ in the style of CDCL, by decisions and propagations.
Theory rules, such as \theory-Propagate and \theory-Conflict, connect this ground search to a theory solver.
Here \ICLF{} abstracts from the concrete theory, or combination of theories, and from the complications involved in theory combination~\cite{DBLP:journals/toplas/NelsonO79,DBLP:journals/jar/TinelliZ05,DBLP:conf/frocos/RaniseRZ05,DBLP:conf/lpar/JovanovicB10}.
Finally, conflict analysis rules derive learned clauses.
The key point is that, although the clauses in $G$ are ground, they remain linked to their non-ground origins in $N$.
During conflict analysis, resolution is therefore performed on these original clauses, not only on their ground instances.

For example, assume that the trail contains, among others, the literals $h(a)=g(a)$, $h(a)=a$, and $g(a)\neq a$.
These literals are inconsistent in the theory of equality, since they contradict transitivity.
In \ICLF{}, this can be detected by applying \theory-Conflict with the closure $x_2\neq x_3\lor x_2\neq x_4\lor x_3=x_4\cdot\{x_2\mapsto h(a), x_3\mapsto g(a),x_4\mapsto a\}$, a theory lemma expressing transitivity of equality.
We write $C\cdot\sigma$ for a clause $C$ together with a grounding substitution $\sigma$.
The instantiated clause $C\sigma$ is used for the ground search, while the associated non-ground clause $C$ is retained for conflict analysis and resolution.
Now assume that $h(a)=g(a)$ was propagated from the closure $c\neq d\lor h(x_1)=g(x_1)\cdot\{x_1\mapsto a\}$.
An \ICLF{} resolution step unifies the contradictory literals $h(x_1)=g(x_1)$ and $x_2\neq x_3$, for instance using the most general unifier $\{x_2\mapsto h(x_1),x_3\mapsto g(x_1)\}$, and derives the clause $c\neq d\lor h(x_1)\neq x_4\lor g(x_1)=x_4$.
Thus, conflict analysis derives a non-ground clause from a ground conflict.

For first-order clause learning, redundancy is a central issue: a learned clause should ideally not be implied by smaller clauses that are already available.
In superposition-based theorem proving~\cite{BachmairGanzinger94b,DBLP:books/el/RV01/NieuwenhuisR01}, redundancy elimination is essential for controlling the search space, but tests such as forward subsumption can be expensive.
Simple Clause Learning (SCL)~\cite{DBLP:conf/cade/FioriW19,DBLP:journals/corr/abs-2302-05954}, a recent approach that lifts ideas from CDCL to first-order logic by guiding non-ground resolution with a ground partial model, avoids this problem in a different way: its learned clauses are non-redundant by construction.
This makes SCL an important starting point for our work.
At the same time, SCL is designed around direct reasoning with non-ground clauses; all non-ground clauses have to be considered for propagations and conflicts and conflict analysis leaves little freedom.
The goal of \ICLF{} is to keep the useful non-redundancy guarantees of SCL while allowing the ground search to remain close to CDCL(T).
Our basic strategy [\Cref{definition:reasonable}] restricts the rules only enough to ensure that learned clauses are non-redundant with respect to the current set $G$ of ground instances [\Cref{theorem:nonRedundant}].
Thus, implementations can focus propagation and conflict search on $G$ and use efficient techniques for ground reasoning, without constantly inspecting the full non-ground clause set $N$.
We also define a stronger, first-order aware strategy [\Cref{definition:FOLaware}], which additionally takes $N$ into account and guarantees non-redundancy with respect to $N$ [\Cref{theorem:nonRedundantN}].
This strategy is closer in spirit to SCL, but still leaves more freedom in conflict analysis.
Finally, \ICLF{} contains multiple backtracking rules.
They capture CDCL/CDCL(T)-style backtracking, first-order backtracking in the spirit of SCL, and a way to incorporate chronological backtracking~\cite{DBLP:conf/sat/NadelR18,DBLP:conf/sat/MohleB19} into SMT solving.

This flexibility makes \ICLF{} a common setting for several forms of model-based reasoning.
In the propositional case, it simulates CDCL.
For first-order logic, it simulates SCL and Resolution.
Moreover, it simulates clause learning in SCL(T)~\cite{DBLP:conf/vmcai/BrombergerFW21}, an adaptation of SCL to the fragment of first-order logic modulo theories without uninterpreted functions and constants.
It also captures CDCL(T)-style reasoning on ground clauses modulo theories.
Finally, \ICLF{} can be seen as a formal setting for studying a range of instantiation-based SMT procedures: although quantifier instantiation is a standard practical approach, it is usually presented through concrete algorithms and solver architectures rather than as one fixed calculus.
In this sense, \ICLF{} is intended as a foundational framework: it identifies which rule restrictions are needed for soundness, non-redundant learning, and termination under suitable restrictions, and which choices can instead be left to strategies and implementations.

The paper is organized as follows.
In \Cref{section:theCalculus}, we present the calculus \ICLF{} and establish, among other properties, soundness~[\Cref{theorem:soundness}], termination under suitable restrictions~[\Cref{theorem:termination}], and completeness under suitable restrictions~[\Cref{corollary:completeness}].
In \Cref{section:simulation}, we show that \ICLF{} simulates the other calculi mentioned above.
This is an extended version of a paper accepted at LPAR 2026~\cite{briefsLPAR2026}.
It contains further examples, lemmas, and explanations in the appendix, as well as all proofs that were omitted from the LPAR paper.

\section{Preliminaries}\label{section:preliminaries}
We consider the following standard notions of many-sorted first-order logic without equality \cite{DBLP:conf/vmcai/BrombergerFW21}.
Let $\Sigma=(\mathcal{S},\Omega,\Pi)$ be a many-sorted signature where $\mathcal{S}$ is a finite, non-empty set of \emph{sort symbols}, $\Omega$ is a non-empty set of \emph{function symbols} over $\mathcal{S}$ and $\Pi$ is a finite, non-empty set of \emph{predicate symbols} over $\mathcal{S}$.
We assume that $\Omega$ contains at least one constant of each sort.
We additionally assume a set of variables $\mathcal{X}$ that contains infinitely many variables of each sort.
First-order terms, atoms, literals and substitutions are defined in the usual way.
We assume that substitutions are well-sorted and that they have finite \emph{domain} $\dom(\sigma)=\{x\mid x\sigma\neq x\}$.
We further define the \emph{codomain} of a substitution as $\codom(\sigma)=\{x\sigma\mid x\in\dom(\sigma)\}$.
A \emph{clause} is a set of literals, which means that duplicate literals are implicitly deleted.
If we write a clause as $L_1\lor\dots\lor L_n$, where the $L_i$ are literals, we mean $\{L_1\}\cup\dots\cup\{L_n\}$.
Likewise, $C\lor L$, where $C$ is a clause and $L$ is a literal, denotes $C\cup\{L\}$, and $C\lor D$, where $C$ and $D$ are clauses, denotes $C\cup D$.
$\bot$ denotes the \emph{empty clause}.
We say that a term is \emph{ground} if it does not contain any variables.
We extend this definition to sets of terms, atoms, literals, clauses and clause sets in the obvious way.
A substitution $\sigma$ is called \emph{grounding} for a term $t$ if $t\sigma$ is a ground term.
We also extend this definition to atoms, literals and clauses in the obvious way.
By $\comp(L)$, we denote the \emph{complement} of a literal $L$.
We further define a function $\atom$ that maps literals to their corresponding atoms.
We extend this function to map clauses to the set of atoms of the literals occurring in the clause and clause sets to the set of atoms of the clauses in the clause set.
A \emph{closure} is a pair $C\cdot\sigma$ where $C$ is a clause or $\top$ and $\sigma$ is a grounding substitution for $C$.
We assume that variables that do not occur in $C$ are implicitly dropped from $\sigma$, i.e., that $\dom(\sigma)$ only contains variables that occur in $C$.
Moreover, we assume that $\codom(\sigma)$ is ground.
If $C\in\{\top,\bot\}$, we may write $\top$ to refer to the closure $\top\cdot\{\}$ and $\bot$ to refer to the closure $\bot\cdot\{\}$.
Given a clause $C$, a closure $C\cdot\sigma$ is called an \emph{instance} of $C$.
For a clause $C$, $\gnd(C)$ denotes the set $\gnd(C)=\{C\sigma\mid\sigma\text{ is grounding for $C$}\}$.
For a clause set $N$, $\gnd(N)=\bigcup_{C\in N}\gnd(C)$.
The function $\mgu$ denotes the \emph{most general unifier} of two terms, atoms or literals.
We assume that the $\mgu$ does not introduce any fresh variables and is idempotent.
Given a ground literal $L$ and a closure $C\cdot\sigma$, let $C=C'\lor L_1\lor\dots\lor L_n$ such that $L_i\sigma=L$ for $i\in\{1,\dots,n\}$ and $L$ does not occur in $C'\sigma$.
Further, let $\mu=\mgu(L_1,\dots,L_n)$.
Then, the result of \emph{exhaustively factorizing} $L$ in $C\cdot\sigma$ is $(C'\lor L_1)\mu\cdot\sigma$.

The semantics of many-sorted first-order logic without equality are given by the notion of an algebra.
A \emph{$\Sigma$-algebra} $\mathcal{A}$ is a mapping that assigns (i) a non-empty carrier set $S^\mathcal{A}$ to every sort $S\in\mathcal{S}$, so that $S_1^\mathcal{A}\cap S_2^\mathcal{A}=\emptyset$ for any distinct sorts $S_1,S_2\in\mathcal{S}$, (ii) a total function $f^\mathcal{A}:S_1^\mathcal{A}\times\dots\times S_n^\mathcal{A}\to S^\mathcal{A}$ to every function symbol $f:S_1\times\dots\times S_n\to S\in\Omega$, (iii) a relation $P^\mathcal{A}\subseteq S_1^\mathcal{A}\times\dots\times S_n^\mathcal{A}$ to every predicate symbol $P\in\Pi$ of arity $n$.
We assume \theory{} to be a \emph{theory}, i.e., a non-empty set of $\Sigma$-algebras.
\theory{} can also be the combination of multiple theories.
For instance, as we consider first-order logic without equality, if needed, equality reasoning has to be part of the theory.
The semantic entailment relation $\models$ is defined in the usual way.
We write $\models_\mathcal{T}$ to denote the semantic entailment relation that considers only the $\Sigma$-algebras $\mathcal{A}\in\mathcal{T}$.

A \emph{trail} is a sequence of annotated, ground literals.
Given a trail $\Gamma$, a ground literal $L$ is called \emph{propositionally true/false} in $\Gamma$ if $L$/$\comp(L)$ occurs in $\Gamma$.
Otherwise, $L$ is called \emph{propositionally undefined} in $\Gamma$.
The trail in \ICLF{} is always \emph{consistent}, i.e., no ground literal is propositionally true and propositionally false at the same time.
A ground clause $C$ is called \emph{propositionally true} in $\Gamma$ if there is a literal $L\in C$ that is propositionally true in $\Gamma$.
It is called \emph{propositionally false} in $\Gamma$ if all literals $L\in C$ are propositionally false in $\Gamma$.
Otherwise, it is called \emph{propositionally undefined} in $\Gamma$.
Trails are interpreted as the conjunction of their literals, clauses are interpreted as the disjunction of their literals and clause sets are interpreted as the conjunction of their clauses.
A clause set $N$ is called \emph{satisfiable} if there exists a $\Sigma$-algebra $\mathcal{A}$ such that $\mathcal{A}\models N$.
It is called \emph{\theory-satisfiable} if there exists a $\Sigma$-algebra $\mathcal{A}\in\mathcal{T}$ such that $\mathcal{A}\models N$.
Similarly, the \theory-satisfiability of a trail $\Gamma$ is defined.
Given a ground set of clauses $N$, a trail $\Gamma$ is called a \emph{model} of $N$ if all clauses in $N$ are propositionally true in $\Gamma$.
Two clause sets $N_1$ and $N_2$ are called \emph{equi-satisfiable} if $N_1$ is satisfiable if and only if $N_2$ is satisfiable.
They are called \emph{equi-satisfiable in \theory{}} if $N_1$ is \theory-satisfiable if and only if $N_2$ is \theory-satisfiable.

\begin{restatable}[Clause Redundancy]{definition}{definitionClauseRedundancy}\label{definition:clauseRedundancy}
	Let $\prec$ be a well-founded, total, strict ordering on ground literals.
	We lift this ordering to clauses by its multiset extension.
	We further define $\preceq$ to be the reflexive closure of $\prec$ and $N^{\preceq C}=\{D\in N\mid D\preceq C\}$ for a ground clause $C$ and a ground clause set $N$.

	A ground clause $C$ is called \emph{redundant} with respect to a ground clause set $N$ and $\prec$ if $N^{\preceq C}\models C$.
	A clause $C$ is called \emph{redundant} with respect to a clause set $N$ and $\prec$ if for all $C'\in\gnd(C)$, it holds that $C'$ is redundant with respect to $\gnd(N)$.
\end{restatable}

\section{The Calculus}\label{section:theCalculus}
We define \ICLF{} as an abstract rewrite system.
Its inference rules operate on a state that is a six-tuple $(\Theta;\Gamma;N;G;k;D\cdot\tau)$ where $\Theta$ is a set of ground atoms, $\Gamma$ is the trail, $N$ is a clause set, $G$ is a set of closures, $k$ is a natural number called the \emph{decision level}, and $D\cdot\tau$ is the \emph{conflict closure}.
We define $\clauses(G)=\{C\sigma\mid C\cdot\sigma\in G\}$.
The purpose of $\Theta$ is to restrict the set of ground atoms that are currently available in the search for a model of $\clauses(G)$.
$\Theta$ always contains at least all atoms that occur on the trail or in $\clauses(G)$ [\Cref{lemma:invariants}].
The literals in $\Gamma$ are either annotated with a natural number, in which case they are called \emph{decisions}, or with a closure, in which case they are called \emph{propagations}, or with \theory, in which case they are called \emph{theory propagations}.
The decisions in $\Gamma$ are annotated with $1,\dots,k$ in this order [\Cref{lemma:invariants}].
The \emph{decision level} of a decision is its annotation, and the decision level of any other literal in $\Gamma$ is the annotation of the decision closest to its left.
If there is no such decision, the decision level of the literal is $0$.
The decision level of $\Gamma$ is the maximal decision level of a literal in it, or $0$ if $\Gamma=\eps$.
For a simpler presentation, $N$ contains both the initial and the learned clauses and $G$ contains both the instantiated and the learned closures.
If $D\cdot\tau=\top$, the algorithm is in the instantiation and model building phase.
If $D\cdot\tau=\bot$, \theory-unsatisfiability has been derived.
Otherwise, the algorithm is in the conflict analysis phase.
The start state for a set of clauses $N'$ is $(\emptyset;\eps;N';\emptyset;0;\top)$.
We assume that the initial set of clauses $N'$ does not contain the empty clause.
Moreover, we require that the \theory-satisfiability of conjunctions of ground literals is decidable.
Instead of using an abstraction function, we write the ground theory literals on the trail directly.
An implementation can still use an abstraction function for the ground reasoning.

We first present the four instantiation rules that are used to manage the instances.

\longerrule{\ICLF{}}{Instantiate}{$(\Theta;\Gamma;N;G;k;\top)$}{$(\Theta\cup\atom(C\sigma);\Gamma;N;G\cup\{C\cdot\sigma\};k;\top)$}{given that $C\in N$, $\sigma$ is grounding for $C$, $C\sigma\notin\clauses(G)$ and either $C\sigma$ is not propositionally false in $\Gamma$ or $\Gamma=\Gamma_1K$ with $K\in\{L^{D\cdot\tau},L^\mathcal{T}\}$ and $C\sigma$ is not propositionally false for $\Gamma_1$.}{20}

Typically, a run starts by applying Instantiate with the empty substitution to all clauses in $N$ that are already ground.
To avoid learning redundant clauses, it is crucial that no instance in $\clauses(G)$ is false for a proper trail prefix or becomes false by a decision [\Cref{definition:reasonable}].
The condition on $C\sigma$ being propositionally false ensures that newly added instances preserve this property.
Allowing propositionally false instances to be instantiated at all is necessary for following a first-order aware strategy [\Cref{definition:FOLaware}].
To add an arbitrary propositionally false instance, Restart can first be used to return to a trail prefix at which Instantiate is applicable.

\shortrule{\ICLF{}}{ClauseDel}{$(\Theta;\Gamma;N\uplus\{C\};G;k;\top)$}{$(\Theta;\Gamma;N;G;k;\top)$}{given that $N\uplus\{C\}$ and $N$ are equi-satisfiable in \theory{} and there is no grounding $\sigma$ with $C\cdot\sigma\in G$.}{20}

Even though it is undecidable in general whether the rule ClauseDel is applicable, there are some sufficient conditions for its applicability, e.g., that $C$ is subsumed by another clause in $N$.

\shortrule{\ICLF{}}{InstanceDel}{$(\Theta;\Gamma;N;G\uplus\{C\cdot\sigma\};k;\top)$}{$(\Theta;\Gamma;N;G;k;\top)$}{given that there is no $L^{C\cdot\sigma}$ in $\Gamma$.}{20}

Although it can make sense to only apply InstanceDel if $G\uplus\{C\cdot\sigma\}$ and $G$ are equi-satisfiable in \theory{}, we would like to allow restarts with fresh instances.

\shortrule{\ICLF{}}{AtomDel}{$(\Theta\uplus\{A\};\Gamma;N;G;k;\top)$}{$(\Theta;\Gamma;N;G;k;\top)$}{given that $A\notin\atom(\Gamma)\cup\atom(\clauses(G))$.}{20}

The following four rules form the ground rules.

\shortrule{\ICLF{}}{Decide}{$(\Theta;\Gamma;N;G;k;\top)$}{$(\Theta;\Gamma L^{k+1};N;G;k+1;\top)$}{given that $\atom(L)\in\Theta$ and $L$ is propositionally undefined in $\Gamma$.}{20}

\shortrule{\ICLF{}}{Propagate}{$(\Theta;\Gamma;N;G;k;\top)$}{$(\Theta;\Gamma L^{C\cdot\sigma};N;G;k;\top)$}{given that $C\cdot\sigma\in G$, $C\sigma=C'\lor L$, $C'$ is propositionally false in $\Gamma$ and $L$ is propositionally undefined in $\Gamma$.}{20}

\shortrule{\ICLF{}}{Conflict}{$(\Theta;\Gamma;N;G;k;\top)$}{$(\Theta;\Gamma;N;G;k;D\cdot\tau)$}{given that $D\cdot\tau\in G$ and $D\tau$ is propositionally false in $\Gamma$.}{20}

\shortrule{\ICLF{}}{Restart}{$(\Theta;\Gamma;N;G;k;\top)$}{$(\Theta;\Gamma_1;N;G;j;\top)$}{given that $\Gamma=\Gamma_1\Gamma_2$ and $j$ is the decision level of $\Gamma_1$.}{20}

Applying the rule Restart with $\Gamma_1=\eps$ corresponds to a restart in CDCL(T), i.e., a restart on the ground level.
We allow restarting with arbitrary trail prefixes, enabling the subsequent application of Instantiate with clauses that would have been propositionally false in $\Gamma$.

Next, we present the five theory rules.

\shortrule{\ICLF{}}{\theory-Propagate}{$(\Theta;\Gamma;N;G;k;\top)$}{$(\Theta;\Gamma L^\mathcal{T};N;G;k;\top)$}{given that $\atom(L)\in\Theta$, $\Gamma\models_\mathcal{T}L$ and $L$ is propositionally undefined in $\Gamma$.}{20}

Theory propagations are annotated with \theory{} and can be explained later when they become relevant during conflict analysis.
The condition $\Gamma\models_\mathcal{T}L$ makes sure that there exists a tautology in the theory from which $L$ can be propagated.

\shortrule{\ICLF{}}{\theory-Learn}{$(\Theta;\Gamma;N;G;k;\top)$}{$(\Theta;\Gamma;N\cup\{C\};G;k;\top)$}{given that $C\notin N$ and $\models_\mathcal{T}C$.}{20}

Note that \theory-Learn adds a non-ground theory lemma in general.
This can then be instantiated by the rule Instantiate.
If the added theory lemma is already ground, it can be instantiated with the empty substitution $\{\}$.

\shortrule{\ICLF{}}{\theory-Atom}{$(\Theta;\Gamma;N;G;k;\top)$}{$(\Theta\cup\{A\};\Gamma;N;G;k;\top)$}{given that $A\notin\Theta$ and $A$ is a ground atom.}{20}

\shortrule{\ICLF{}}{\theory-Conflict}{$(\Theta;\Gamma;N;G;k;\top)$}{$(\Theta;\Gamma;N;G;k;D\cdot\tau)$}{given that $\models_\mathcal{T}D$, $\tau$ is grounding for $D$ and $D\tau$ is propositionally false in $\Gamma$.}{20}

Just like in \theory-Learn, the conflict clause produced by \theory-Conflict can be non-ground.

\longerrule{\ICLF{}}{Explain}{$(\Theta;\Gamma_1 L^\mathcal{T}\Gamma_2;N;G;k;D\cdot\tau)$}{$(\Theta;\Gamma_1 L^{C\cdot\sigma}\Gamma_2;N;G;k;D\cdot\tau)$}{given that $\models_\mathcal{T}C$, $\sigma$ is grounding for $C$, $C\sigma=C'\lor L$ and $C'$ is propositionally false in $\Gamma_1$.}{20}

$\top$ is a short form for the conflict closure $\top\cdot\{\}$, i.e., the rule Explain can be applied both in the model building and in the conflict analysis phase.
Just like in \theory-Learn and \theory-Conflict, the reason can be non-ground.
For each theory propagation on the trail, there exists a clause to apply Explain with, for example $\neg\Gamma_1\lor L\cdot\{\}$, which is always valid in the theory by \Cref{lemma:invariants}.

Finally, the following five rules are the conflict analysis rules.

\longerrule{\ICLF{}}{Resolve}{$(\Theta;\Gamma_1 L^{C\cdot\sigma}\Gamma_2;N;G;k;(D\lor L')\cdot\tau)$}{$(\Theta;\Gamma_1 L^{C\cdot\sigma}\Gamma_2;N;G;k;\resolve(C\cdot\sigma, D\cdot\tau, L', L))$}{given that $L'\tau=\comp(L)$.}{20}

The function $\resolve$ takes as arguments two variable disjoint closures $C\cdot\sigma$ and $D\cdot\tau$, a literal $L'$ that does not occur in $D$, may share variables with $D$, but not with $C$, and a ground literal $L$ such that $L'\tau=\comp(L)$.
Let $C=C'\lor L_1\lor\dots\lor L_n$ such that $L_i\sigma=L$ for all $i\in\{1,\dots,n\}$ and $L$ does not occur in $C'\sigma$.
If $\mu=\mgu(\comp(L'),L_1,\dots,L_n)$, then $\resolve(C\cdot\sigma,D\cdot\tau,L',L)=(C'\lor D)\mu\cdot\sigma\tau$.
In other words, the function $\resolve$ exhaustively factorizes $L$ in $C$ and computes a resolution inference between the two clauses.

The rule Resolve allows resolving any literal that is on the trail.
This generality makes standard simplifications such as unit reduction possible.
Additionally, it allows modeling conflict analysis procedures such as in standard CDCL with 1UIP learning~\cite{DBLP:series/faia/0001LM21} or in SCL~\cite{DBLP:journals/corr/abs-2302-05954} that work through the trail backwards, performing either Skip or Resolve for each encountered literal, as well as procedures like chronological backtracking \cite{DBLP:conf/sat/NadelR18,DBLP:conf/sat/MohleB19} that only resolve with a subsequence of the relevant literals on the trail.

\longerrule{\ICLF{}}{Factorize}{$(\Theta;\Gamma;N;G;k;(D\lor L\lor L')\cdot\tau)$}{$(\Theta;\Gamma;N;G;k;(D\lor L)\eta\cdot\tau)$}{given that $L\tau=L'\tau$ and $\eta=\mgu(L, L')$.}{20}

The ground conflict clause does not change by applying Factorize.
If the literal that Resolve is applied to can be factorized, then it does not disappear from the ground conflict clause after resolving.
While the resulting ground conflict clause is the same regardless of whether a literal is exhaustively factorized before it is resolved or whether Resolve is applied multiple times, this is, in general, not true for the non-ground conflict clause \cite[Example 3]{DBLP:journals/corr/abs-2302-05954}.

\longerrule{\ICLF{}}{BacktrackClassic}{$(\Theta;\Gamma_1K^j\Gamma_2;N;G;k;D\cdot\tau)$}{$(\Theta;\Gamma_1L^{D\cdot\tau};N\cup\{D\};G\cup\{D\cdot\tau\};j-1;\top)$}{given that $D\tau=D'\lor L$, $D'$ is propositionally false in $\Gamma_1$ and $L$ is propositionally undefined in $\Gamma_1$.}{20}

\longerrule{\ICLF{}}{BacktrackFOL}{$(\Theta;\Gamma_1\Gamma_2;N;G;k;D\cdot\tau)$}{$(\Theta;\Gamma_1;N\cup\{D\};G\cup\{D\cdot\tau\};j;\top)$}{given that there is no grounding $\sigma$ such that $D\sigma$ is propositionally false in $\Gamma_1$, $j$ is the decision level of $\Gamma_1$ and $\Gamma_2$ contains a decision literal.}{20}

\longerrule{\ICLF{}}{BacktrackCB}{$(\Theta;\Gamma;N;G;k;D\cdot\tau)$}{$(\Theta;\Gamma'L^{D\cdot\tau};N\cup\{D\};G\cup\{D\cdot\tau\};j;\top)$}{given that $\Gamma'=\cb(\Gamma, D\tau, j)$, $D\tau=D'\lor L$, $D'$ is propositionally false in $\Gamma'$ and $L$ is propositionally undefined in $\Gamma'$.}{20}

Here, CB stands for chronological backtracking~\cite{DBLP:conf/sat/NadelR18,DBLP:conf/sat/MohleB19}, an alternative backtracking scheme for CDCL that we adapt to \ICLF{} in this paper.
We now define the function $\cb$.
To this end, given a trail $\Gamma$, we first define the \emph{conflict graph} as the directed acyclic graph whose nodes are the literals on the trail $\Gamma$.
Decision literals $L^k$ and theory propagations $L^\mathcal{T}$ do not have any predecessors, and the predecessors of a propagation $L^{C'\cdot\sigma}$ are the literals $\comp(L')$ for all $L'\in C'\sigma$ with $L'\neq L$.
If $L$ is a propagation (or an explained theory propagation), all of these predecessors are indeed literals on the trail.
For each literal $L$ on the trail $\Gamma$, we define $\reach(L)$ as the set of literals that can be reached from $L$ in the conflict graph, including $L$ itself.
By $\reach^{-1}(L)$, we denote the set of literals that can reach $L$ in the conflict graph.

The arguments of the function $\cb$ are a trail $\Gamma$, a ground clause $C$ and a destination decision level $j$ that is smaller than the decision level of $\Gamma$.
Let $\Gamma=\Gamma_1L'^{j+1}\Gamma_2$.
$\cb$ is only applicable if the following three conditions are met.
First, $C$ should be propositionally false in $\Gamma$, i.e., for all $L\in C$, $\comp(L)$ should be in $\Gamma$.
Second, $|\{\comp(L)\mid L\in C\}\cap\reach(L')|=1$, i.e., exactly one literal in $C$ should depend on the $j+1$-st decision.
Third, in $\bigcup_{L\in C}\reach^{-1}(\comp(L))\cap\Gamma_2$, there should be no literal annotated as a decision or theory propagation, i.e., no conflict literal should depend on a decision or theory propagation after $L'$ on the trail.
Then, let $C=C'\lor L''$ such that $\{\comp(L)\mid L\in C'\}\cap\reach(L')=\emptyset$, i.e., $L''$ is the only literal in $C$ that depends on $L'$.
Further, let $\Gamma_3$ be the subsequence of $\Gamma_2$ that contains only the literals $\bigcup_{L\in C'}\reach^{-1}(\comp(L))\cap\Gamma_2$.
It is crucial to note that $\Gamma_3$ only contains propagations (and explained theory propagations), and that none of the literals in $\Gamma_3$ was propagated from $L'$.
Then $\cb(\Gamma,C,j)=\Gamma_1\Gamma_3$.
By the aforementioned properties of $\Gamma_3$, all propagations in $\Gamma_3$ are still justified in $\Gamma_1\Gamma_3$ and $C'$ stays propositionally false in $\Gamma_1\Gamma_3$.
\Cref{figure:illustrationCB} illustrates this definition of the function $\cb$.

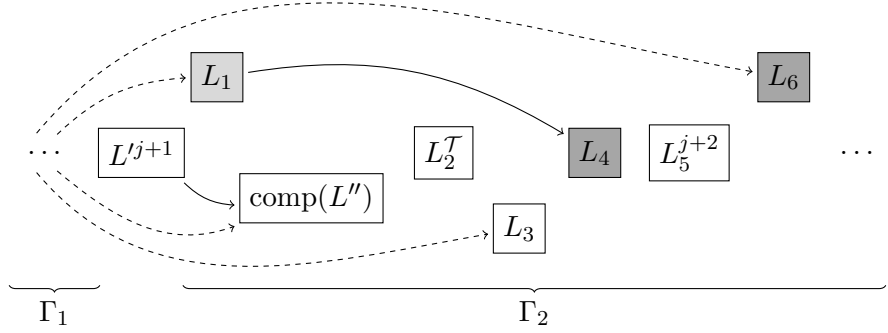
\begin{figure}[t]
	\begin{center}
		\begin{tikzpicture}[scale=0.5]
			\node[draw,minimum height=.65cm,outer sep=2pt] (l) at (0,0) {$L'^{j+1}$};
			\node[draw,minimum height=.65cm,outer sep=2pt,fill=black!15!white] (l1) at (2,2) {$L_1$};
			\node[draw,minimum height=.65cm,outer sep=2pt] (cl) at (4.5,-1.2) {$\comp(L'')$};
			\node[draw,minimum height=.65cm,outer sep=2pt] (l2) at (8,0) {$L_2^\mathcal{T}$};
			\node[draw,minimum height=.65cm,outer sep=2pt] (l3) at (10,-2) {$L_3$};
			\node[draw,minimum height=.65cm,outer sep=2pt,fill=black!35!white] (l4) at (12,0) {$L_4$};
			\node[draw,minimum height=.65cm,outer sep=2pt] (l5) at (14.5,0) {$L_5^{j+2}$};
			\node[draw,minimum height=.65cm,outer sep=2pt,fill=black!35!white] (l6) at (17,2) {$L_6$};
			\node (d1) at (-2.5,0) {$\dots$};
			\node (d2) at (19,0) {$\dots$};
			\node (tmp1) at (-2.5,.25) {};
			\node (tmp2) at (-2.5,-.25) {};
			\node (tmp3) at (-3.,-.25) {};
			\node (tmp4) at (-3,.25) {};
			\draw [->] (l1) edge[bend left=20] (l4);
			\draw [->] (l) edge[bend right=20] (cl);
			\draw [dash pattern=on 2pt off 2pt,->] (tmp1) edge[bend left=20] (l1);
			\draw [dash pattern=on 2pt off 2pt,->] (tmp3) edge[out=310,in=190] (l3);
			\draw [dash pattern=on 2pt off 2pt,->] (tmp2) edge[out=320,in=200] (cl);
			\draw [dash pattern=on 2pt off 2pt,->] (tmp4) edge[out=50,in=170] (l6);

			\draw [decorate,decoration={brace,mirror}] (-3.5,-3.5) -- (-1.1,-3.5) node[pos=0.5,anchor=north,yshift=-2pt] {$\Gamma_1$};
			\draw [decorate,decoration={brace,mirror}] (1.1,-3.5) -- (19.7,-3.5) node[pos=0.5,anchor=north,yshift=-2pt] {$\Gamma_2$};
		\end{tikzpicture}
	\end{center}
	\caption{An example to illustrate the definition of the function $\cb$. $C'$ contains the complements of the literals $L_4$ and $L_6$. Then $\Gamma_3=L_1L_4L_6$.} \label{figure:illustrationCB}
\end{figure}

\begin{restatable}{definition}{definitionReasonable}\label{definition:reasonable}
	A strategy for \ICLF{} is called \emph{reasonable} if the following are true:
	\begin{enumerate}
		\item The rule Conflict is preferred over the rules Decide, Propagate, \theory-Propagate and \theory-Conflict.
		\item No application of the rule Decide enables an immediate application of the rule Conflict.
		\item If a conflict state is reached by the rule Conflict, then Resolve and Factorize are applied in such a way that the complement of the rightmost literal on the trail does not occur in the ground conflict clause anymore before applying any of the rules BacktrackClassic, BacktrackFOL or BacktrackCB.
	\end{enumerate}
\end{restatable}

These requirements do not cause any stuck states by \Cref{theorem:stuckStates}.
Although this does not affect correctness or termination, to avoid that previously resolved literals are reintroduced into the conflict clause, Resolve should be applied from right to left on the trail.

The goal of restricting the strategy is to learn only clauses that are non-redundant with respect to the current set of instantiations $G$ and a trail induced ordering [\Cref{definition:trailOrdering}].
Most SAT solvers implement CDCL with 1UIP learning \cite{DBLP:series/faia/0001LM21}.
In CDCL, the 1UIP clause is non-redundant because exhaustive propagation maintains the necessary invariants~\cite{DBLP:conf/birthday/Weidenbach15}.
In \ICLF{}, however, these invariants can be broken by theory rules and instantiations.
For example, a newly instantiated clause may propagate with respect to an earlier trail prefix, making it possible for the 1UIP clause to be redundant.
Moreover, for theory propagations and for literals propagated from them, the exact implying literals are not known before a theory explanation is computed, which may be expensive.
For these reasons, 1UIP learning is not straightforward in our setting.
Therefore, the restrictions of a reasonable strategy capture only what is needed for non-redundant learning and otherwise leave the conflict analysis rules general.

The following example shows \ICLF{} in action modulo the theory of equality for the sake of simple theory reasoning.
Of course, \ICLF{} works modulo any suitable theory combination in the same way as CDCL(T).
This run closely resembles quantifier instantiation combined with CDCL(T) and shows how \ICLF{} can model this standard approach.

\begingroup
\allowdisplaybreaks
\begin{example}\label{example:big}
	We present a reasonable \ICLF{} run on the clause set
	\begin{align*}
		N_0=\{&(N1)~f(x_0)=g(x_0)\lor f(x_1)=h(x_1),&&(N2)~g(b)=c,\\
		&(N3)~f(a)\neq g(a)\lor c=b,&&(N4)~f(b)\neq h(b)\lor f(c)\neq h(c),\\
		&(N5)~f(a)\neq g(a)\lor g(x_2)\neq x_2\}.
	\end{align*}
	\begin{align*}
		&(\emptyset;\eps;N_0;\emptyset;0;\top)\\
		\ruleapp{\ICLF{}}{Instantiate $\times5$}&(\Theta_0;\eps;N_0;G_0;0;\top)\\
		\intertext{Here, $\Theta_0=\{f(a)=g(a),f(b)=h(b),g(b)=c,c=b,f(c)=h(c),g(b)=b\}$ and
		\begin{align*}
			G_0=\{&(G1)~f(x_0)=g(x_0)\lor f(x_1)=h(x_1)\cdot\{x_0\mapsto a,x_1\mapsto b\},\\
			&(G2)~g(b)=c\cdot\{\},\\
			&(G3)~f(a)\neq g(a)\lor c=b\cdot\{\},\\
			&(G4)~f(b)\neq h(b)\lor f(c)\neq h(c)\cdot\{\},\\
			&(G5)~f(a)\neq g(a)\lor g(x_2)\neq x_2\cdot\{x_2\mapsto b\}\}.
		\end{align*}}\\[-7.2ex]
		\ruleapp{\ICLF{}}{Propagate}&(\Theta_0;\Gamma_1:=(g(b)=c)^{(G2)};N_0;G_0;0;\top)\\
		\ruleapp{\ICLF{}}{Decide}&(\Theta_0;\Gamma_2:=\Gamma_1(f(a)=g(a))^1;N_0;G_0;1;\top)\\
		\ruleapp{\ICLF{}}{Propagate}&(\Theta_0;\Gamma_3:=\Gamma_2(c=b)^{(G3)};N_0;G_0;1;\top)\\
		\ruleapp{\ICLF{}}{Propagate}&(\Theta_0;\Gamma_4:=\Gamma_3(g(b)\neq b)^{(G5)};N_0;G_0;1;\top)\\
		\ruleapp{\ICLF{}}{\theory-Conflict}&(\Theta_0;\Gamma_4;N_0;G_0;1;g(b)\neq c\lor c\neq b\lor g(b)=b\cdot\{\})\\
		\ruleapp{\ICLF{}}{Resolve$(G5)$}&(\Theta_0;\Gamma_4;N_0;G_0;1;f(a)\neq g(a)\lor g(b)\neq c\lor c\neq b\cdot\{\})\\
		\ruleapp{\ICLF{}}{Resolve$(G3)$}&(\Theta_0;\Gamma_4;N_0;G_0;1;f(a)\neq g(a)\lor g(b)\neq c\cdot\{\})\\
		\ruleapp{\ICLF{}}{BacktrackClassic}&(\Theta_0;\Gamma_5:=(g(b)=c)^{(G2)}(f(a)\neq g(a))^{(G6)};N_1;G_1;0;\top)\\
		\intertext{Here, $N_1=N_0\cup\{(N6)~f(a)\neq g(a)\lor g(b)\neq c\}$ and $G_1=G_0\cup\{(G6)~f(a)\neq g(a)\lor g(b)\neq c\cdot\{\}\}$.}
		\ruleapp{\ICLF{}}{Propagate}&(\Theta_0;\Gamma_6:=\Gamma_5(f(b)=h(b))^{(G1)};N_1;G_1;0;\top)\\
		\ruleapp{\ICLF{}}{Propagate}&(\Theta_0;\Gamma_7:=\Gamma_6(f(c)\neq h(c))^{(G4)};N_1;G_1;0;\top)
	\end{align*}
	$\Gamma_7$ is a \theory-satisfiable model of $G_1$.
	A possible next step is to apply Instantiate with the clause $(N1)\cdot\{x_0\mapsto a,x_1\mapsto c\}$.
	This immediately enables the application of Conflict, and after several applications of Resolve, $\bot$ will be derived, showing the unsatisfiability of the clause set $N_0$.
\end{example}
\endgroup

The following example shows how \ICLF{} can learn non-ground clauses.
Again, we work modulo the theory of equality to keep the theory reasoning simple.

\begingroup
\allowdisplaybreaks
\begin{example}\label{example:nonground}
	We present a reasonable \ICLF{} run on the clause set
	\begin{align*}
		N_0=\{(N1)~c=d,~~~~~(N2)~c\neq d\lor h(x_0)=x_0,~~~~~(N3)~c\neq d\lor h(x_1)=g(x_1)\}.
	\end{align*}
	\begin{align*}
		&(\emptyset;\eps;N_0;\emptyset;0;\top)\\
		\ruleapp{\ICLF{}}{Instantiate $\times3$}&(\Theta_0;\eps;N_0;G_0;0;\top)\\
		\intertext{Here, $\Theta_0=\{c=d,~h(a)=a,~h(a)=g(a)\}$ and
			$G_0=\{(G1)~c=d,~(G2)~c\neq d\lor h(x_0)=x_0$ $\cdot\,\{x_0\mapsto a\},~(G3)~c\neq d\lor h(x_1)=g(x_1)\cdot\{x_1\mapsto a\}\}.$
		}
		\ruleapp{\ICLF{}}{Propagate}&(\Theta_0;\Gamma_1:=(c=d)^{(G1)};N_0;G_0;0;\top)\\
		\ruleapp{\ICLF{}}{Propagate}&(\Theta_0;\Gamma_2:=\Gamma_1(h(a)=a)^{(G2)};N_0;G_0;0;\top)\\
		\ruleapp{\ICLF{}}{Propagate}&(\Theta_0;\Gamma_3:=\Gamma_2(h(a)=g(a))^{(G3)};N_0;G_0;0;\top)\\
		\ruleapp{\ICLF{}}{\theory-Atom}&(\Theta_1:=\Theta_0\cup\{g(a)=a\};\Gamma_3;N_0;G_0;0;\top)\\
		\ruleapp{\ICLF{}}{Decide}&(\Theta_1;\Gamma_4:=\Gamma_3(g(a)\neq a)^1;N_0;G_0;1;\top)\\
		\ruleapp{\ICLF{}}{\theory-Conflict}&(\Theta_1;\Gamma_4;N_0;G_0;1;x_2\neq x_3\lor x_2\neq x_4 \lor x_3=x_4\\[-1.6ex]\intertext{\hfill $\cdot\,\{x_2\mapsto h(a), x_3\mapsto g(a), x_4\mapsto a\})$}\\[-5.1ex]
		\ruleapp{\ICLF{}}{Resolve$(G3)$}&(\Theta_1;\Gamma_4;N_0;G_0;1;c\neq d\lor h(x_1)\neq x_4\lor g(x_1)=x_4\\[-1.6ex]\intertext{\hfill $\cdot\,\{x_1\mapsto a, x_4\mapsto a\})$}\\[-5.1ex]
		\ruleapp{\ICLF{}}{Resolve$(G2)$}&(\Theta_1;\Gamma_4;N_0;G_0;1;c\neq d\lor g(x_4)=x_4\cdot\{x_4\mapsto a\})\\
		\ruleapp{\ICLF{}}{Resolve$(G1)$}&(\Theta_1;\Gamma_4;N_0;G_0;1;g(x_4)=x_4\cdot\{x_4\mapsto a\})\\
		\ruleapp{\ICLF{}}{BacktrackClassic}&(\Theta_1;\Gamma_3(g(a)=a)^{(G4)};N_1;G_1;0;\top)
	\end{align*}
	Here, $N_1=N_0\cup\{(N4)~g(x_4)=x_4\}$ and $G_1=G_0\cup\{(G4)~g(x_4)=x_4\cdot\{x_4\mapsto a\}\}$.
	Of course, after the third propagation, $\Gamma_3$ is already a \theory-satisfiable model of $G_0$.
	However, consider a larger problem that contains the three clauses $N1$, $N2$, and $N3$ among other clauses.
	In such a problem, another clause might have motivated the decision $g(a)\neq a$.
	A CDCL(T)-style approach with quantifier instantiation and ground 1UIP conflict analysis would learn the ground theory conflict clause $h(a)\neq g(a)\lor h(a)\neq a\lor g(a)=a$ from this conflict.
	By contrast, in this particular run, the generality of Resolve allows \ICLF{} to resolve further with the three propagations and learn the more general clause $g(x_4)=x_4$.
	For these further resolution steps to yield the non-ground clause $g(x_4)=x_4$, the non-ground theory lemma $x_2\neq x_3\lor x_2\neq x_4\lor x_3=x_4$ is essential: starting from its ground instance instead, the same resolution steps would only yield the ground clause $g(a)=a$.
	The learned non-ground clause can be useful later in the run: if $g(b)$ becomes relevant, it can be instantiated as $g(x_4)=x_4\cdot\{x_4\mapsto b\}$, whereas the ground clause $g(a)=a$ carries no information about $g(b)$.
\end{example}
\endgroup

Next, we prove the correctness of \ICLF{} and further properties, including non-redundant learning as well as termination and completeness under suitable restrictions.
We start with invariants that are maintained by all \ICLF{} rules.
The proof is by induction on the length of the derivation, with a case distinction on the last applied rule.

\begin{restatable}[Invariants]{lemma}{lemmaInvariants}\label{lemma:invariants}
	Let $(\Theta;\Gamma;N;G;k;D\cdot\tau)$ be a state reached from a start state $(\emptyset;\eps;N';\emptyset;0;\top)$ by applying the rules of \ICLF{}. Then the following hold:
	\begin{enumerate}
		\item For all $C\cdot\sigma\in G$, it holds that $C\in N$.\label{property:Ginstances}
		\item $\atom(\Gamma)\cup\atom(\clauses(G))\subseteq\Theta$ and for all $A\in\Theta$, $A$ is a ground atom.\label{property:thetaContains}
		\item If $\Gamma=\Gamma_1K\Gamma_2$ then $K$ is propositionally undefined in $\Gamma_1$.\label{property:gammaUnique}
		\item If $\Gamma=\Gamma_1L^{C\cdot\sigma}\Gamma_2$ then $C\sigma$ is ground, $C\sigma=C'\lor L$, $C'$ is propositionally false in $\Gamma_1$ and either $C\cdot\sigma\in G$ or $\models_\mathcal{T} C$.\label{property:gammaPropagations}
		\item If $\Gamma=\Gamma_1 L^\mathcal{T}\Gamma_2$ then $\Gamma_1\models_\mathcal{T} L$.\label{property:gammaT}
		\item Let $k_1,\dots,k_l$ be the annotations of decisions in $\Gamma$. If there are no decisions, then $k=0$; otherwise $1=k_1$, $k=k_l$ and for all $i\in\{2,\dots,l\}$, $k_i=k_{i-1}+1$.\label{property:gammaLevels}
		\item $N'$ and $N$ are equi-satisfiable in \theory.\label{property:satPreservation}
		\item If $D\neq\top$ then $D\tau$ is ground, $\atom(D\tau)\subseteq\Theta$ and $D\tau$ is propositionally false in $\Gamma$.\label{property:conflictFalse}
		\item $\clauses(G)\models_\mathcal{T}D\tau$.\label{property:conflictFollowsG}
		\item $N\models_\mathcal{T}D$.\label{property:conflictFollowsN}
	\end{enumerate}
\end{restatable}
\newcommand{\lemmaInvariantsProof}{
\begin{proof}
	The proof is by induction on the length of the derivation.
	In the start state $(\emptyset;\eps;N';\emptyset;0;\top)$, all properties obviously hold.
	The induction step is shown by a case distinction on the last applied rule.
	For readability, we only mention rules that modify any of the objects occurring in the respective property and skip cases in which the property obviously follows from the induction hypothesis.
	When we say that a rule asserts something, we mean that it follows from the premise and the patterns on the left side of the rule.
	To avoid confusion between variables used in this lemma and variables used in the rule definitions, we mark variables that refer to the rule definitions with a hat, e.g., $\hat{C}$.

	\Cref{property:Ginstances,property:thetaContains}:
	Instantiate asserts all of these properties for the newly added instance.
	ClauseDel asserts that \Cref{property:Ginstances} stays true after deleting the clause.
	AtomDel asserts that \Cref{property:thetaContains} stays true after deleting the atom.
	Decide asserts that \Cref{property:thetaContains} stays true after adding the literal to the trail.
	For Propagate, \Cref{property:thetaContains} for the newly added literal follows from the induction hypothesis because $\hat{C}\cdot\hat{\sigma}\in G$.
	\theory-Propagate asserts that \Cref{property:thetaContains} stays true after adding the literal to the trail.
	BacktrackClassic, BacktrackFOL and BacktrackCB ensure \Cref{property:Ginstances} by adding $\hat{D}$ to $\hat{N}$ and $\hat{D}\cdot\hat{\tau}$ to $\hat{G}$ simultaneously.
	\Cref{property:thetaContains} follows from \Cref{property:conflictFalse} of the induction hypothesis.
	\theory-Atom asserts that \Cref{property:thetaContains} stays true after adding the atom to $\Theta$.
	
	\Cref{property:gammaUnique,property:gammaPropagations,property:gammaT,property:gammaLevels}:
	InstanceDel asserts that \Cref{property:gammaPropagations} stays true after deleting the instance.
	Decide asserts that \Cref{property:gammaUnique} stays true after adding the literal to the trail.
	\Cref{property:gammaLevels} follows from the induction hypothesis together with the level $k+1$ that we assign to the added literal.
	Propagate asserts that \Cref{property:gammaUnique,property:gammaPropagations} stay true after adding the literal to the trail.
	For Restart, all of these properties follow from the induction hypothesis because $\hat\Gamma_1$ is a prefix of $\hat\Gamma$ and the decision level is set to be correct.
	\theory-Propagate asserts that \Cref{property:gammaUnique,property:gammaT} stay true after adding the literal to the trail.
	Explain asserts that \Cref{property:gammaPropagations} stays true after changing the annotation of the literal.
	BacktrackClassic and BacktrackCB assert that \Cref{property:gammaUnique,property:gammaPropagations} stay true after adding the propagation from the learned clause.
	BacktrackClassic and BacktrackFOL assert that \Cref{property:gammaLevels} stays true.
	The trail after BacktrackCB is $\hat\Gamma_1\hat\Gamma_3\hat L^{\hat D\cdot\hat\tau}$, where $\hat\Gamma_1\hat\Gamma_3=\cb(\hat\Gamma,\hat D\hat\tau,\hat j)$.
	As argued along the definition of the function $\cb$, $\hat\Gamma_3$ only contains propagations (and explained theory propagations), none of which depends on $\hat{L'}$, and all their implying literals are retained.
	Hence \Cref{property:gammaT} follows from the induction hypothesis, and \Cref{property:gammaPropagations} follows from the induction hypothesis for the retained propagations and from the premise of BacktrackCB for the newly added propagation.
	\Cref{property:gammaLevels} also stays true by the definition of the function $\cb$.

	\Cref{property:satPreservation}:
	ClauseDel asserts that this property stays true after deleting the clause.
	When applying \theory-Learn, the equi-satisfiability in \theory{} of $\hat{N}$ and $\hat{N}\cup\{\hat{C}\}$ follows from the fact that $\models_\mathcal{T}\hat{C}$.
	For BacktrackClassic, BacktrackFOL and BacktrackCB, this property follows from \Cref{property:conflictFollowsN} of the induction hypothesis.

	\Cref{property:conflictFalse,property:conflictFollowsG,property:conflictFollowsN}:
	For all rules that leave a state with $D\cdot\tau=\top$, these properties are obviously true.
	For Conflict, \Cref{property:conflictFalse} follows from \Cref{property:thetaContains} of the induction hypothesis and the premise of this rule.
	\Cref{property:conflictFollowsG} is obvious and \Cref{property:conflictFollowsN} follows from \Cref{property:Ginstances} of the induction hypothesis.
	For \theory-Conflict, \Cref{property:conflictFalse} follows from the premise of this rule and \Cref{property:thetaContains} of the induction hypothesis; if $D\tau$ is propositionally false in $\Gamma$, then all of its literals are contained in $\Gamma$ negatively.
	\Cref{property:conflictFollowsG,property:conflictFollowsN} are obvious because $\models_\mathcal{T}D$.
	For Resolve, all of these properties follow from \Cref{property:Ginstances,property:thetaContains,property:gammaPropagations,property:conflictFalse,property:conflictFollowsG,property:conflictFollowsN} of the induction hypothesis and the definition of the function $\resolve$.
	Factorize trivially maintains \Cref{property:conflictFalse,property:conflictFollowsG} because $D\tau$ does not change by factorizing the literals $\hat{L}$ and $\hat{L'}$ with $\hat{L}\hat{\tau}=\hat{L'}\hat{\tau}$.
	\Cref{property:conflictFollowsN} follows from the induction hypothesis because $(\hat{D}\lor\hat{L}\lor\hat{L'})\models(\hat{D}\lor\hat{L})\hat\mu$.
\end{proof}
}
\ifthenelse{\boolean{proofsInPaper}}{\lemmaInvariantsProof}{}

The soundness of \ICLF{} follows immediately from \Cref{property:satPreservation,property:conflictFollowsN}.

\begin{restatable}[Soundness]{theorem}{theoremSoundness}\label{theorem:soundness}
	If $(\emptyset;\eps;N';\emptyset;0;\top)\shortruleapp{\ICLF{}}{$\ast$}(\Theta;\Gamma;N;G;k;\bot)$, then $N'\models_\mathcal{T}\bot$.
\end{restatable}
\newcommand{\theoremSoundnessProof}{
\begin{proof}
	This follows from \Cref{property:satPreservation,property:conflictFollowsN} of \Cref{lemma:invariants}.
\end{proof}
}
\ifthenelse{\boolean{proofsInPaper}}{\theoremSoundnessProof}{}

The two previous statements hold regardless of the strategy used to apply the rules of \ICLF{}.
The following properties require the rules to be applied according to a reasonable strategy, as defined in \Cref{definition:reasonable}.
We first show two invariants maintained by a reasonable strategy.
The first invariant, \Cref{lemma:preferConflict}, states that there is no propositional conflict for a proper trail prefix, and if the last literal on the trail is a decision, then there is no propositional conflict.
This invariant is used to prove non-redundant learning with respect to $G$ in \Cref{theorem:nonRedundant}; it implies that, after conflict analysis, the learned clause is the only false clause for a suitable trail prefix [\Cref{lemma:helpNonRedundant}].
It also helps us show that whenever Conflict is applicable, the last literal on the trail is a propagation or a theory propagation whose complement occurs in the conflict clause [\Cref{lemma:conflictOnPropagation}].
This is essential for satisfying the third requirement of a reasonable strategy, which requires resolving with the last literal on the trail before backtracking.
The lemma is therefore used to show that there are no stuck states in \Cref{theorem:stuckStates}.

\begin{restatable}{lemma}{lemmaPreferConflict}\label{lemma:preferConflict}
	If $(\emptyset;\eps;N';\emptyset;0;\top)\shortruleapp{\ICLF{}}{$\ast$}(\Theta;\Gamma';N;G;k;D\cdot\tau)$ using a reasonable strategy, then if $\Gamma'=\Gamma K$, there is no $C\cdot\sigma\in G$ such that $C\sigma$ is propositionally false in $\Gamma$, and if $K$ is a decision, then there is no $C\cdot\sigma\in G$ such that $C\sigma$ is propositionally false in $\Gamma K$.
\end{restatable}
\newcommand{\lemmaPreferConflictProof}{
\begin{proof}
	The proof is by induction on the length of the derivation.
	The statement is obviously true for the start state.
	The induction step is shown by a case distinction on the last applied rule.
	For readability, we skip rules that do not modify the trail and do not insert into $G$ because for these rules, the statement trivially follows from the induction hypothesis.

	Instantiate ensures in its premise that the new instance satisfies the condition.
	For Decide, this follows from the fact that a reasonable strategy is used:
	If there is a $C\cdot\sigma$ in $G$ such that $C\sigma$ is propositionally false in $\Gamma$, then Conflict would have been applicable instead of Decide, and if there is a $C\cdot\sigma$ in $G$ such that $C\sigma$ is propositionally false in $\Gamma K$, then Conflict would be applicable next, contradicting the fact that in a reasonable strategy, no application of the rule Decide enables an immediate application of the rule Conflict.
	For Propagate and \theory-Propagate, we only have to argue about the case where an instance is propositionally false in $\Gamma$, and this follows from the same argument we used for Decide.
	For BacktrackClassic, for all clauses that were already contained in $G$ before applying this rule, the condition follows from the induction hypothesis because a suffix of the trail is deleted and then a decision is replaced by a propagation.
	For the learned clause, this follows from the premise of this rule.
	For BacktrackFOL, the argument is very similar, except that no propagation is added to the trail.
	For BacktrackCB, we actually need the third requirement of a reasonable strategy.
	During conflict analysis, the literals on the trail, $N$ and $G$ are not modified by any rule.
	If the conflict was reached by \theory-Conflict, then no instance in $G$ was propositionally false because otherwise, Conflict would have been applicable.
	If the conflict was reached by Conflict, then by the induction hypothesis, no instance in $G$ was propositionally false for a proper trail prefix.
	In a reasonable run, when BacktrackCB is applied, then the complement of the rightmost literal on the trail does not occur in the conflict clause.
	Thus, the rightmost literal on the trail is not contained in the subsequence of literals that is kept by the function $\cb$.
	From this, the statement follows for all instances that were previously contained in $G$, and for the learned clause, it follows by the premise of this rule.
\end{proof}
}
\ifthenelse{\boolean{proofsInPaper}}{\lemmaPreferConflictProof}{}

\begin{restatable}{lemma}{lemmaConflictOnPropagation}\label{lemma:conflictOnPropagation}
	If $(\emptyset;\eps;N';\emptyset;0;\top)\shortruleapp{\ICLF{}}{$\ast$}(\Theta;\Gamma;N;G;k;\top)$ using a reasonable strategy and Conflict is applicable to this state with $D\cdot\tau\in G$, then $\Gamma=\Gamma_1K$, $K$ is not a decision literal and $D\tau$ is not propositionally false in $\Gamma_1$.
\end{restatable}
\newcommand{\lemmaConflictOnPropagationProof}{
\begin{proof}
	By assumption, $N'$ does not contain the empty clause initially, and the empty clause is not valid in the theory, hence \theory-Learn cannot add it.
	The empty clause can also not be learned by any of the backtrack rules, so we can assume that $N$ does not contain the empty clause.
	Thus, by \Cref{property:Ginstances} of \Cref{lemma:invariants}, $G$ does not contain the empty clause.
	From this, it follows that $\Gamma\neq\eps$, and the statement follows from \Cref{lemma:preferConflict}.
\end{proof}
}
\ifthenelse{\boolean{proofsInPaper}}{\lemmaConflictOnPropagationProof}{}

Now we are ready to show that \ICLF{} learns only non-redundant clauses when a reasonable strategy is used.
First, we define an ordering with respect to which the learned clauses are non-redundant.

\begin{restatable}[Trail Induced Ordering]{definition}{definitionTrailOrdering}\label{definition:trailOrdering}
	Given a trail $\Gamma=L_1\dots L_n$, we define the trail induced ordering $\prec_\Gamma$ as a well-founded, total and strict ordering in which $L_1\prec_\Gamma\comp(L_1)\prec_\Gamma L_2\prec_\Gamma\comp(L_2)\prec_\Gamma\dots\prec_\Gamma L_n\prec_\Gamma\comp(L_n)$ and all literals that are propositionally undefined in $\Gamma$ are larger than $\comp(L_n)$.
\end{restatable}

\begin{restatable}{lemma}{lemmaHelpNonRedundant}\label{lemma:helpNonRedundant}
	If $(\emptyset;\eps;N';\emptyset;0;\top)\shortruleapp{\ICLF{}}{$\ast$}(\Theta;\Gamma;N;G;k;D\cdot\tau)$ using a reasonable strategy and BacktrackClassic, BacktrackFOL or BacktrackCB is applicable to this state following a reasonable strategy, then $\Gamma=\Gamma_1\Gamma_2$, $D\tau$ is propositionally false in $\Gamma_1$ and no clause in $\clauses(G)$ is propositionally false in $\Gamma_1$.
\end{restatable}
\newcommand{\lemmaHelpNonRedundantProof}{
\begin{proof}
	If any backtrack rule is applicable to this state, then $D\notin\{\top,\bot\}$.
	This means that previously, either Conflict or \theory-Conflict must have been applied, and that after the most recent application of Conflict or \theory-Conflict, only the rules Explain, Resolve and Factorize have been applied.
	Notice that none of these rules modifies the literals on the trail (Explain only modifies the annotations), and hence the literals on the trail are still the same as right before Conflict or \theory-Conflict was applied.
	Moreover, none of these rules modifies $\Theta$, $N$, $G$ or $k$.
	If the conflict state was reached by Conflict, then the complement of the rightmost literal on the trail does not occur in $D\tau$ because a reasonable strategy was used.
	Hence, by \Cref{property:conflictFalse} of \Cref{lemma:invariants}, it follows that $\Gamma=\Gamma_1K$ and $D\tau$ is false in $\Gamma_1$.
	Now, the claim follows from \Cref{lemma:preferConflict}.
	If the conflict state was reached by \theory-Conflict, then there is no $C\cdot\sigma$ in $G$ that is propositionally false in $\Gamma$ because otherwise, Conflict would have been applicable instead, and the claim follows.
\end{proof}
}
\ifthenelse{\boolean{proofsInPaper}}{\lemmaHelpNonRedundantProof}{}

\begin{restatable}[Non-Redundant Learning]{theorem}{theoremNonRedundant}\label{theorem:nonRedundant}
	Whenever \ICLF{} learns a ground clause (using BacktrackClassic, BacktrackFOL or BacktrackCB) in a reasonable run, this clause is non-redundant with respect to the current set of ground instances and the trail induced ordering.
\end{restatable}
\newcommand{\theoremNonRedundantProof}{
\begin{proof}
	Assume that \ICLF{} is in the state $(\Theta;\Gamma;N;G;k;D\cdot\tau)$ that BacktrackClassic, BacktrackFOL or BacktrackCB is applicable to.
	By \Cref{lemma:helpNonRedundant}, $\Gamma=\Gamma_1\Gamma_2$, $D\tau$ is propositionally false in $\Gamma_1$ and no clause in $\clauses(G)$ is propositionally false in $\Gamma_1$.
	This means that in particular, all literals of $D\tau$ are propositionally defined in $\Gamma_1$.
	Consider the set $\clauses(G)^{\preceq_\Gamma D\tau}$.
	All literals of the clauses in this set must be propositionally defined in $\Gamma_1$ as well because they are all $\preceq_\Gamma D\tau$.
	As no clause in $\clauses(G)$ is propositionally false in $\Gamma_1$, it holds that $\Gamma_1\models\clauses(G)^{\preceq_\Gamma D\tau}$.
	It follows that $\clauses(G)^{\preceq_\Gamma D\tau}\not\models D\tau$ because otherwise, by transitivity of entailment, it would follow that $\Gamma_1\models D\tau$, contradicting the fact that $D\tau$ is propositionally false in $\Gamma_1$.
	Hence, $D\tau$ is non-redundant with respect to $\clauses(G)$ and the trail induced ordering $\prec_\Gamma$.
\end{proof}
}
\ifthenelse{\boolean{proofsInPaper}}{\theoremNonRedundantProof}{}

The non-ground learned clause $D$, however, can be redundant with respect to $N$ and the trail-induced ordering $\prec_\Gamma$.
This can happen, for example, if there is an instance of a clause in $N$ that is propositionally false for a trail prefix, but that is not contained in $G$; one such case is shown in~\Cref{example:nonRedundantG}.
\Cref{theorem:nonRedundantN} shows that if a stronger strategy, called first-order aware strategy [\Cref{definition:FOLaware}], is followed, then the non-ground learned clause $D$ is also non-redundant with respect to $N$.
Moreover, \Cref{theorem:stuckStatesN} shows that following a first-order aware strategy does not cause any stuck states.

\begingroup
\allowdisplaybreaks
\begin{example}\label{example:nonRedundantG}
	Consider the clause set
	\begin{align*}
		N_0=\{(N1)~P(x_0),~~~~~(N2)~P(x_1)\lor R(x_1),~~~~~(N3)~Q(x_2)\lor\neg R(x_2)\}.
	\end{align*}
	\begin{align*}
		&(\emptyset;\eps;N_0;\emptyset;0;\top)\\
		\ruleapp{\ICLF{}}{Instantiate $\times2$}&(\Theta_0;\eps;N_0;G_0;0;\top)\\
		\intertext{Here, $\Theta_0=\{P(a),R(a),Q(a)\}$ and $G_0=\{(G1)~P(x_1)\lor R(x_1)\cdot\{x_1\mapsto a\},(G2)~Q(x_2)\lor\neg R(x_2)\cdot\{x_2\mapsto a\}\}$.}
		\ruleapp{\ICLF{}}{Decide $\times2$}&(\Theta_0;\neg P(a)^1\neg Q(a)^2;N_0;G_0;2;\top)\\
		\ruleapp{\ICLF{}}{Propagate}&(\Theta_0;\neg P(a)^1\neg Q(a)^2 R(a)^{(G1)};N_0;G_0;2;\top)\\
		\ruleapp{\ICLF{}}{Conflict}&(\Theta_0;\neg P(a)^1\neg Q(a)^2 R(a)^{(G1)};N_0;G_0;2;Q(x_2)\lor\neg R(x_2)\cdot\{x_2\mapsto a\})\\
		\ruleapp{\ICLF{}}{Resolve}&(\Theta_0;\neg P(a)^1\neg Q(a)^2 R(a)^{(G1)};N_0;G_0;2;P(x_2)\lor Q(x_2)\cdot\{x_2\mapsto a\})
	\end{align*}
	The clause $P(x_2)\lor Q(x_2)$ that can now be learned is non-redundant with respect to $G_0$, but redundant with respect to $N_0$ due to the clause $P(x_0)$.
	This also illustrates the difference between a reasonable and a first-order aware strategy.
	A first-order aware strategy prevents the decision $\neg P(a)$ since it makes an instance of $P(x_0)\in N_0$ propositionally false, whereas a reasonable strategy allows this decision because $P(a)\notin\clauses(G_0)$.
\end{example}
\endgroup

\begin{restatable}{corollary}{corollaryGunique}\label{corollary:gUnique}
	If $(\emptyset;\eps;N';\emptyset;0;\top)\shortruleapp{\ICLF{}}{$\ast$}(\Theta;\Gamma;N;G;k;D\cdot\tau)$ using a reasonable strategy, then there are no two $C\cdot\sigma,C'\cdot\sigma'\in G$ such that $C\cdot\sigma\neq C'\cdot\sigma'$, but $C\sigma=C'\sigma'$.
\end{restatable}
\newcommand{\corollaryGuniqueProof}{
\begin{proof}
	The proof is by induction on the length of the derivation.
	The statement is obviously true for the start state.
	For the induction step, we only have to consider rules that insert instances into $G$.
	Instantiate asserts this property in its premise, and for the backtrack rules, it follows from \Cref{theorem:nonRedundant}.
\end{proof}
}
\ifthenelse{\boolean{proofsInPaper}}{\corollaryGuniqueProof}{}

Without restrictions, termination cannot be expected for \ICLF{}: Instantiate may generate infinitely many ground instances, \theory-Learn may add infinitely many theory lemmas, and \theory-Atom may add arbitrary ground atoms.
Even ground reasoning on a fixed set of ground instances may diverge due to Restart and InstanceDel, the latter being analogous to Forget in CDCL(T).
We therefore prove termination when these rules are applied only finitely often.
Such restrictions are standard in related settings.
For example, SCL restricts trail literals to be smaller than a fixed bound $\beta$, thereby obtaining only finitely many relevant instances~\cite{DBLP:journals/corr/abs-2302-05954}.
In CDCL, restarts can be performed with increasing periodicity to ensure termination~\cite{DBLP:journals/jacm/NieuwenhuisOT06}.
The following theorem states the corresponding termination condition for \ICLF{}.

\begin{restatable}[Termination]{theorem}{theoremTermination}\label{theorem:termination}
	If \ICLF{} is executed from a start state $(\emptyset;\eps;N';\emptyset;0;\top)$ using a reasonable strategy, and the rules Instantiate, Instance\-Del, Restart, \theory-Learn, and \theory-Atom are applied only finitely often, then it terminates.
\end{restatable}
\newcommand{\theoremTerminationProof}{
\begin{proof}
	Assume $(\emptyset;\eps;N';\emptyset;0;\top)\shortruleapp{\ICLF{}}{$\ast$}(\Theta;\Gamma;N;G;k;D\cdot\tau)$.
	It is enough to show that there is no infinite sequence of reasonable rule applications from this state if the rules Instantiate, InstanceDel, Restart, \theory-Learn and \theory-Atom are not applied.
	For this purpose, we define a function $\numT(\hat\Gamma)$ that counts the number of \theory{} annotations on a trail $\hat\Gamma$, and a function $\multiGround(\hat D\cdot\hat\tau)$ that maps closures to the multiset $\multiGround(\hat D\cdot\hat\tau)=\{|L\hat\tau\mid L\in\hat D|\}$.
	We now associate states $(\hat\Theta;\hat\Gamma;\hat N;\hat G;\hat k;\hat D\cdot\hat\tau)$ with the five tuple
	\begin{align*}
		\left(4^{\big|\hat\Theta\big|}-\big|\clauses(\hat G)\big|,~~\big|\hat N\big|,~~\big|\hat\Theta\big|-\big|\hat\Gamma\big|,~~\numT\big(\hat\Gamma\big),~~\multiGround(\hat D\cdot\hat\tau)\right).
	\end{align*}
	By \Cref{property:thetaContains,property:gammaUnique} of \Cref{lemma:invariants}, the first and third component are natural numbers.
	For the second and fourth component, it is obvious that they are natural numbers.
	The fifth component is compared by the multiset extension of the trail induced ordering $\prec_{\hat\Gamma}$, which is well-founded.
	We assume that $\multiGround(\top)$ is maximal in this ordering.
	It suffices to show that each of the remaining rules decreases this five tuple with respect to the lexicographic ordering.
	The fact that the trail $\hat\Gamma$, and hence the trail induced ordering $\prec_{\hat\Gamma}$, are not constant is not a problem, as changes of the trail $\hat\Gamma$ always coincide with a decrease in an earlier component.

	For readability, we only write for each rule which component it decreases, and omit the sentence that the earlier components remain unchanged.
	ClauseDel decreases the second component.
	AtomDel decreases the first component.
	Decide and Propagate decrease the third component.
	Conflict decreases the fifth component.
	\theory-Propagate decreases the third component.
	\theory-Conflict decreases the fifth component.
	Explain decreases the fourth component.
	Resolve and Factorize decrease the fifth component.
	By \Cref{theorem:nonRedundant}, BacktrackClassic, BacktrackFOL and BacktrackCB decrease the first component.
\end{proof}
}
\ifthenelse{\boolean{proofsInPaper}}{\theoremTerminationProof}{}

\begin{restatable}[No Stuck States]{theorem}{theoremStuckStates}\label{theorem:stuckStates}
	Let $(\Theta;\Gamma;N;G;k;D\cdot\tau)$ be a state reached by \ICLF{} from a start state $(\emptyset;\eps;N';\emptyset;0;\top)$ using a reasonable strategy.
	Then either $D=\bot$, or $D=\top$ and $\Gamma$ is a \theory-satisfiable model of $\clauses(G)$, or a rule from the set $\{$Decide, Propagate, Conflict, \theory-Conflict, Explain, Resolve, \texttt{B}$\}$ is applicable following a reasonable strategy, where \texttt{B} is any of the rules BacktrackClassic, BacktrackFOL or BacktrackCB.
\end{restatable}
\newcommand{\theoremStuckStatesProof}{
\begin{proof}
	First assume that $D=\top$, but $\Gamma$ is not a \theory-satisfiable model of $\clauses(G)$.
	If there is an instance $C\cdot\sigma\in G$ such that $C\sigma$ is propositionally false in $\Gamma$, then Conflict is applicable.
	If there is some $A\in\Theta$ such that $A$ is propositionally undefined in $\Gamma$, then either Propagate or Decide are applicable with $A$ or $\neg A$.
	Otherwise, by \Cref{property:thetaContains} of \Cref{lemma:invariants}, all instances in $G$ are propositionally true, which means that by our assumption, $\Gamma$ is not \theory-satisfiable.
	It follows that \theory-Conflict is applicable, for instance with $\neg\Gamma\cdot\{\}$.

	Now assume that $D\notin\{\top,\bot\}$.
	By \Cref{property:gammaT} of \Cref{lemma:invariants}, for any literal $L$ on the trail such that $\Gamma=\Gamma_1L^\mathcal{T}\Gamma_2$, Explain is applicable, for instance with $\neg\Gamma_1\lor L$.
	Hence, we can assume that we can apply Resolve on any literal on the trail that is not a decision.
	By \Cref{property:conflictFalse} of \Cref{lemma:invariants}, the complements of all literals in $D\tau$ are on the trail.
	We now show that for each backtrack rule, there exists an approach to be able to apply it.
	If the respective backtrack rule is not applicable due to the restrictions of the reasonable strategy, then Resolve can be applied with the rightmost literal on the trail, whose complement in this case occurs in the conflict clause by \Cref{lemma:conflictOnPropagation}.
	It is possible to resolve with this rightmost literal on the trail until its complement does not occur in the conflict clause anymore by \Cref{property:gammaPropagations} of \Cref{lemma:invariants} and \Cref{theorem:termination}.
	If BacktrackClassic is not applicable, then the rightmost complement of a literal in $D\tau$ on the trail is not a decision, so Resolve can be applied with it.
	Following this process, either $\bot$ will be derived or BacktrackClassic becomes applicable eventually.
	The same strategy works for BacktrackCB because if the rightmost complement of a literal in $D\tau$ is a decision $L'^{j+1}$, then all conditions for the applicability of $\cb$ are met.
	If there is no decision on the trail $\Gamma$, then $\bot$ can be derived by applying Resolve several times.
	Otherwise, BacktrackFOL is always applicable with $\Gamma_1=\eps$.
\end{proof}
}
\ifthenelse{\boolean{proofsInPaper}}{\theoremStuckStatesProof}{}

In full generality, completeness cannot be expected for arbitrary theories and theory combinations, for example for linear arithmetic combined with uninterpreted function symbols~\cite{DBLP:journals/aaecc/BachmairGW94,Waldmann2001IJCAR,DBLP:conf/csl/KorovinV07}.
Nevertheless, \Cref{theorem:termination,theorem:stuckStates} imply completeness whenever a finite \theory-unsatisfiable set of ground instances is available.

\begin{restatable}[Completeness]{corollary}{corollaryCompleteness}\label{corollary:completeness}
	Let $N$ be a clause set and $S\subseteq\gnd(N)$ be finite and \theory-unsatisfiable.
	Then \ICLF{} can derive $\bot$.
\end{restatable}
\newcommand{\corollaryCompletenessProof}{
\begin{proof}
	We first apply Instantiate with all clauses in $S$, which is possible by the definition of $\gnd(N)$ and the fact that the trail is empty in the start state.
	Now, the rules Decide, Propagate, Conflict, \theory-Conflict, Explain, Resolve and BacktrackClassic are applied exhaustively following a reasonable strategy.
	None of these rules deletes clauses from $G$.
	By \Cref{theorem:termination}, applying these rules exhaustively terminates.
	By \Cref{theorem:stuckStates}, $\bot$ must have been derived because $S\subseteq\clauses(G)$ is \theory-unsatisfiable, so there cannot be a \theory-satisfiable model of $\clauses(G)$.
\end{proof}
}
\ifthenelse{\boolean{proofsInPaper}}{\corollaryCompletenessProof}{}

\begin{restatable}{definition}{definitionFOLaware}\label{definition:FOLaware}
	A strategy for \ICLF{} is called \emph{first-order aware} if it is reasonable and the following are true:
	\begin{enumerate}
		\item If Instantiate is applicable with a propositionally false instance, then this is always preferred over applying \theory-Conflict.
		\item At no point during the run is there a clause in the clause set that has an instance that is propositionally false for a proper trail prefix.
		\item Whenever the rightmost literal on the trail is a decision, no clause in the clause set has an instance that is propositionally false on the trail.
	\end{enumerate}
\end{restatable}

\begin{restatable}{theorem}{theoremNonRedundantN}\label{theorem:nonRedundantN}
	If $(\emptyset;\eps;N';\emptyset;0;\top)\shortruleapp{\ICLF{}}{$\ast$}(\Theta;\Gamma;N;G;k;D\cdot\tau)$ using a first-order aware strategy and one of the rules BacktrackClassic, BacktrackFOL and BacktrackCB is applicable to this state following a first-order aware strategy, then the clause $D$ added to $N$ by the respective backtrack rule is non-redundant with respect to $N$ and the trail induced ordering $\prec_\Gamma$.
\end{restatable}
\newcommand{\theoremNonRedundantNProof}{
\begin{proof}
	We show the non-redundancy of $D$ by showing that $D\tau\in\gnd(D)$ is non-redundant with respect to $\gnd(N)$.
	Consider the set $\gnd(N)^{\preceq_\Gamma D\tau}$.
	Analogously to the proof of \Cref{lemma:helpNonRedundant}, we distinguish whether the conflict state was reached by Conflict or \theory-Conflict.

	First consider the case that the conflict state was reached by the rule Conflict.
	Then $\Gamma=\Gamma_1K$ and $D\tau$ is propositionally false in $\Gamma_1$.
	Due to the fact that a first-order aware strategy is used, no clause in $\gnd(N)$ can be propositionally false in $\Gamma_1$.
	Nevertheless, all literals of the clauses in $\gnd(N)^{\preceq_\Gamma D\tau}$ are propositionally defined in $\Gamma_1$, which means that $\Gamma_1\models\gnd(N)^{\preceq_\Gamma D\tau}$, from which, analogously to the proof of \Cref{theorem:nonRedundant}, it follows that $\gnd(N)^{\preceq_\Gamma D\tau}\not\models D\tau$, hence, $D$ is non-redundant with respect to $N$ and $\prec_\Gamma$.

	Now consider the case that the conflict state was reached by the rule \theory-Conflict.
	We argue that no clause in $\gnd(N)$ can be propositionally false in $\Gamma$.
	From this, the claim follows analogously to the case that the conflict state was reached by Conflict.
	Clearly, $\Gamma=\Gamma_1K$.
	If $K$ is a propagation or theory propagation, then no clause in $\gnd(N)$ can be propositionally false in $\Gamma$ because otherwise Conflict, possibly after Instantiate, would have been applied in a first-order aware strategy instead of \theory-Conflict.
	If $K$ is a decision, then this follows directly from the fact that a first-order aware strategy is used.
\end{proof}
}
\ifthenelse{\boolean{proofsInPaper}}{\theoremNonRedundantNProof}{}

\begin{restatable}{theorem}{theoremStuckStatesN}\label{theorem:stuckStatesN}
	If $(\emptyset;\eps;N';\emptyset;0;\top)\shortruleapp{\ICLF{}}{$\ast$}(\Theta;\Gamma;N;G;k;D\cdot\tau)$ using a first-order aware strategy, then either $D=\bot$, or $D=\top$, $\Gamma$ is a \theory-satisfiable model of $\clauses(G)$ and no instance of a clause in $N$ is propositionally false in $\Gamma$, or a rule from the set $\{$Instantiate, Decide, Propagate, Conflict, \theory-Conflict, Explain, Resolve, BacktrackFOL$\}$ is applicable following a first-order aware strategy.
	Additionally, Instantiate only has to be applied with instances whose atoms are already contained in $\Theta$.
\end{restatable}
\newcommand{\theoremStuckStatesNProof}{
\begin{proof}
	It is obvious that the start state is valid in a first-order aware strategy.
	First assume that $D=\top$, but $\Gamma$ is not a \theory-satisfiable model of $\clauses(G)$ or there is an instance of a clause in $N$ that is propositionally false in $\Gamma$.
	Further assume that there is a $C\in N$ and a grounding $\sigma$ such that $C\sigma$ is propositionally false in $\Gamma$.
	If $C\sigma\in\clauses(G)$, then Conflict is applicable.
	Otherwise, due to the fact that a first-order aware strategy is used, the rightmost literal in $\Gamma$ is a propagation or theory propagation, so Instantiate is applicable.
	Due to the fact that $C\sigma$ is propositionally false in $\Gamma$, all of its atoms are indeed contained in $\Theta$ by \Cref{property:thetaContains} of \Cref{lemma:invariants}.
	Now consider the case that no instance of a clause in $N$ is propositionally false in $\Gamma$.
	If there is some $A\in\Theta$ that is propositionally undefined in $\Gamma$, then either Propagate or Decide are applicable with $A$ or $\neg A$.
	In the case of Propagate, it is not possible that the first-order aware strategy is violated because by assumption, there is no propositionally false instance in $\Gamma$, which becomes the longest proper trail prefix after applying Propagate.
	If the application of Decide would cause a propositionally false instance, then either this instance can be used to propagate the complement instead, or Instantiate can be applied instead.
	Again, clearly all atoms of this instance are already contained in $\Theta$.
	If all $A\in\Theta$ are propositionally defined in $\Gamma$, then all instances in $G$ must be propositionally true, which means that by our assumption, $\Gamma$ is not \theory-satisfiable.
	\theory-Conflict can be applied in this case, for example, with $\neg\Gamma\cdot\{\}$.
	This does not violate the first-order aware strategy because by assumption, no instance is propositionally false in $\Gamma$.

	Now assume that $D\notin\{\top,\bot\}$.
	Analogously to the proof of \Cref{theorem:stuckStates}, either $\bot$ can be derived or a state can be reached in which BacktrackFOL is applicable according to a reasonable strategy.
	This is because none of the rules Explain, Resolve and Factorize modifies the literals on the trail or the clause set.
	Now BacktrackFOL can be applied to go back to any trail prefix $\Gamma_1$ for which $D$ has no false instance and such that the removed suffix contains a decision; in particular, $\Gamma_1=\eps$ is possible.
\end{proof}
}
\ifthenelse{\boolean{proofsInPaper}}{\theoremStuckStatesNProof}{}

It should be noted that although following a first-order aware strategy does not cause any stuck states [\Cref{theorem:stuckStatesN}], it limits the applicability of some rules.
For instance, Propagate and \theory-Propagate cannot be applied if there is an uninstantiated false instance; instead, this false instance has to be instantiated.
Additionally, Decide cannot be applied if doing so would make some instance false.
In this sense, the requirements of a first-order aware strategy in \ICLF{} are very similar to the requirements of a regular run in SCL(FOL)~\cite{DBLP:journals/corr/abs-2302-05954}.
Moreover, the first-order aware strategy also restricts the applicability of \theory-Learn.
If a clause that has a propositionally false instance should be added by \theory-Learn, Restart can first be used to go back to a point on the trail at which no instance is propositionally false.
Beyond that, while detecting conflicts and propagations with respect to ground clause sets such as $G$ is straightforward, both tasks are NP-complete in the worst case with respect to first-order clause sets such as $N$; nevertheless, the recent lifting of the two-watched literal scheme to first-order logic shows that such clause sets can still be handled efficiently in practice~\cite{DBLP:conf/ijcar/BriefsBGLSW26}.

\section{Simulation of Other Calculi}\label{section:simulation}
In this section, we show that \ICLF{} can simulate various other calculi.
To avoid confusion, we mark variables from the simulated calculi with an overline, e.g., $\overline N$, and refer to variables from \ICLF{} without any overlines, e.g., by $N$.

First, we show that \ICLF{} can simulate SCL(FOL) \cite{DBLP:journals/corr/abs-2302-05954}.
This means that we associate each possible SCL(FOL) state with a set of \ICLF{} states.
Then, we show that for each state in a given SCL(FOL) run, we can find an associated \ICLF{} state and \ICLF{} rule applications between these states.
The simulation is linear in the sense that for each SCL(FOL) rule application, only a constant number of \ICLF{} rule applications has to be performed.
SCL(FOL) is an algorithm for first-order logic without equality, so we assume \theory{} to be the \emph{free} theory, i.e., the theory that contains all $\Sigma$-algebras.
For SCL(FOL), we write the status closure as $\overline D\cdot\overline\tau$, assuming that $\top=\top\cdot\{\}$ and $\bot=\bot\cdot\{\}$, just as we do for \ICLF{}.
Let $\prec_B$ be a well-founded, total and strict ordering on ground atoms such that for any ground atom $A$, there are only finitely many ground atoms $A'$ with $A'\prec_BA$.
$\prec_B$ is lifted to literals by comparing the respective atoms and to clauses by its multiset extension.

\begin{restatable}{definition}{definitionAssociatedSCL}\label{definition:associatedSCL}
	We say that an \ICLF{} state $(\Theta;\Gamma;N;G;k;D\cdot\tau)$ is \emph{associated} with a given SCL(FOL) state $(\overline{\Gamma};\overline{N};\overline{U};\overline{\beta};\overline{k};\overline{D}\cdot\overline\tau)$ if all of the following  hold:
	\begin{enumerate}
		\item $D\cdot\tau=\overline D\cdot\overline\tau$.
		\item For all $A\in\Theta$, $A\prec_B\overline\beta$ and there is an atom $A'\in\atom(N)$ such that $A$ is a ground instance of $A'$.
		\item $N=\overline N\cup\overline U$.
		\item $k\geq\overline k$. If $\overline D\cdot\overline\tau=\top$, then $k=\overline k$.
		\item If $\Gamma=K_1\dots K_n$ and $\overline\Gamma=\overline K_1\dots\overline K_{\overline n}$, then $n\geq\overline n$ and for all $i\in\{1,\dots,\overline n\}$, $K_i$ and $\overline K_i$ match.
			By match, we mean that either $K_i=\overline K_i=L^j$ or that $K_i=L^{C\cdot\sigma}$, $\overline K_i=L^{\overline C\cdot\overline\sigma}$ and $\overline C\cdot\overline\sigma$ is the result of exhaustively factorizing $L$ in $C\cdot\sigma$.
			If $\overline D\cdot\overline\tau=\top$, then $n=\overline n$.
	\end{enumerate}
\end{restatable}

\begin{restatable}[\ICLF{} Simulates SCL(FOL)]{theorem}{theoremSimulationSCL}\label{theorem:simulationSCL}
	Let $N'$ be a first-order clause set and $\beta'$ be a ground literal.
	Then for any regular SCL(FOL) run $$(\eps;N';\emptyset;\beta';0;\top)=S_0\shortruleapp{SCL(FOL)}{}S_1\shortruleapp{SCL(FOL)}{}\dots\shortruleapp{SCL(FOL)}{}S_m,$$
	there exists a reasonable \ICLF{} run $$(\emptyset;\eps;N';\emptyset;0;\top)=T_0\shortruleapp{\ICLF{}}{$\leq3$}T_1\shortruleapp{\ICLF{}}{$\leq3$}\dots\shortruleapp{\ICLF{}}{$\leq3$}T_m$$ such that for all $i\in\{0,\dots,m\}$, $T_i$ is associated with $S_i$.
	Here, $\shortruleapp{\ICLF{}}{$\leq3$}$ means that at most three \ICLF{} rules are applied to reach the following state.
	Only the \ICLF{} rules Instantiate, InstanceDel, Decide, Propagate, Conflict, \theory-Atom, Resolve, Factorize and BacktrackFOL are needed.
\end{restatable}
\newcommand{\theoremSimulationSCLProof}{
\begin{proof}
	The proof is by induction on $m$.
	The start state $(\emptyset;\eps;N';\emptyset;0;\top)$ of \ICLF{} is obviously associated with $(\eps;N';\emptyset;\beta';0;\top)$, the start state of SCL(FOL).
	Now let $S_{m'}=(\overline{\Gamma};\overline{N};\overline{U};\overline{\beta};\overline{k};\overline{D}\cdot\overline\tau)$ and $T_{m'}=(\Theta;\Gamma;N;G;k;D\cdot\tau)$.
	We show the induction step by a case distinction on the SCL(FOL) rule that is applied to reach $S_{m'+1}$ from $S_{m'}$.
	For readability, we do not mention properties that neither the applied SCL(FOL) rule nor any of the applied \ICLF{} rules change, as these trivially follow from the induction hypothesis.
	We first argue that the first two requirements of a reasonable strategy in \ICLF{} do not restrict the simulation.
	If the rule Conflict is applicable in \ICLF{}, then it is also applicable for any associated SCL(FOL) state because the first-order clause that the conflict clause is an instance of is also contained in the SCL(FOL) state by \Cref{property:Ginstances} of \Cref{lemma:invariants} and the fact that $N=\overline N\cup\overline U$.
	A regular SCL(FOL) run follows the same requirements about preferring Conflict over Propagate and Decide and Decide not enabling Conflict as a reasonable \ICLF{} run.
	Hence, whenever Propagate and Decide can be applied in a regular SCL(FOL) run, they can also be applied in a reasonable \ICLF{} run.

	Propagate:
	Let $C\lor L$ be the clause used for propagation and $\sigma$ be the grounding such that $L\sigma$ is the propagated literal.
	Since $C\lor L\in(\overline N\cup\overline U)$, $C\lor L\in N$.
	If $(C\lor L)\cdot \sigma\in G$, then Propagate can be applied in \ICLF{} right away, otherwise Instantiate has to be applied first.
	If Instantiate is not applicable because $(C\lor L)\sigma\in\clauses(G)$, then InstanceDel can be applied with this other instance because by \Cref{property:gammaPropagations} of \Cref{lemma:invariants}, this other instance cannot be an annotation of a literal on the trail because otherwise, Propagate would not be applicable with $(C\lor L)\sigma$.
	By \Cref{corollary:gUnique}, there is at most one such instance in $G$ that has to be deleted.
	It is not possible that the new instance is propositionally false in $\Gamma$ because the literals in $\Gamma$ and $\overline\Gamma$ are the same and Conflict would have been applied in a regular SCL(FOL) run instead.
	The atoms added to $\Theta$ by Instantiate are smaller than $\overline\beta$ because $(C\lor L)\sigma\prec_B\{\overline\beta\}$ by the premise of SCL(FOL) Propagate, and they are obviously ground instances of atoms in $\atom(N)=\atom(\overline N\cup\overline U)$.
	The propagated literal is $L\sigma$ in both cases and the annotation in SCL(FOL) is the result of exhaustively factorizing the annotation in \ICLF{}.

	Decide:
	Let $L$ be a literal occurring in $\overline N\cup\overline U$ and $L\sigma$ be the decided literal.
	If $\atom(L\sigma)\in\Theta$, then Decide can be applied in \ICLF{} right away, otherwise \theory-Atom can be applied with $\atom(L\sigma)$ first.
	The atom added to $\Theta$ is smaller than $\overline\beta$ and a ground instance of an atom in $\atom(N)=\atom(\overline N\cup\overline U)$ by the premise of SCL(FOL) Decide.
	Both SCL(FOL) Decide and \ICLF{} Decide add the same ground literal $L\sigma$ with the same annotation $k+1=\overline{k}+1$ to the trail.

	Conflict:
	Let $C\in\overline N\cup\overline U$ be the conflicting clause and $\sigma$ be the grounding for $C$.
	If $C\cdot\sigma$ is contained in $G$, then Conflict can be applied in \ICLF{} immediately, otherwise Instantiate has to be applied first.
	If Instantiate is not applicable because $C\sigma\in\clauses(G)$, then InstanceDel must be applied with the other instance first, which is possible because $C\sigma$ is false in $\Gamma$, so the other instance cannot be an annotation of a literal on the trail.
	By \Cref{corollary:gUnique}, there is at most one such instance in $G$ that has to be deleted.
	When Conflict is applicable in a regular SCL(FOL) run, then the rightmost literal on the trail is a propagation that occurs negatively in the conflict clause \cite[Lemma 11]{DBLP:journals/corr/abs-2302-05954}.
	Hence, by \Cref{property:gammaUnique} of \Cref{lemma:invariants}, the remaining requirements of Instantiate are met.
	After applying Conflict, the conflict closures in SCL(FOL) and \ICLF{} are the same.

	Skip:
	$T_{m'}$ is also associated with $S_{m'+1}$, so no \ICLF{} rule has to be applied.

	Factorize:
	The SCL(FOL) rule Factorize can be simulated by the \ICLF{} rule Factorize.
	
	Resolve:
	The SCL(FOL) rule Resolve can be simulated by the \ICLF{} rule Resolve.
	Propagation annotations in SCL(FOL) are already exhaustively factorized, whereas \ICLF{} factorizes them in Resolve.

	Backtrack:
	The SCL(FOL) rule Backtrack can be simulated by the \ICLF{} rule BacktrackFOL.
	As the conflict clauses are the same, the SCL(FOL) trail is a prefix of the \ICLF{} trail in the state that Backtrack is applied to and SCL(FOL) skips over a decision, \ICLF{} can go back to the same point on the trail that SCL(FOL) goes back to.
	What is left to argue is that BacktrackFOL can be applied in a reasonable strategy.
	In the simulation, we use the \ICLF{} rule Conflict only when SCL(FOL) uses its Conflict rule, and whenever SCL(FOL) resolves, \ICLF{} also resolves in the simulation.
	Since in SCL(FOL), the rightmost literal on the trail is always resolved exhaustively during conflict resolution \cite[Lemma 11]{DBLP:journals/corr/abs-2302-05954}, BacktrackFOL is applicable in a reasonable strategy in \ICLF{} whenever Backtrack is applicable in SCL(FOL).

	Grow:
	As the new bound is $\prec_B$-larger than the old bound $\overline\beta$, $T_{m'}$ is also associated with $S_{m'+1}$, so no \ICLF{} rule has to be applied.
\end{proof}
}
\ifthenelse{\boolean{proofsInPaper}}{\theoremSimulationSCLProof}{}

A further lemma in the appendix shows that if SCL(FOL) reaches a stuck state, then the associated \ICLF{} state reached in the corresponding simulation is also stuck according to a suitable definition of stuckness [\Cref{lemma:SCLStuckSimulation}].

SCL(FOL) has been shown to simulate non-redundant ground superposition~\cite{BrombergerJW23,BrombergerDW25}.
Since \ICLF{} simulates SCL(FOL), this result also transfers to \ICLF{}.
For \ICLF{}, however, we can show a stronger result: it can simulate any resolution inference that does not yield a tautology, provided that neither parent clause becomes tautological under the most general unifier of the resolved literals.
If a parent clause does become tautological, the resolvent is already subsumed by the other parent clause.
This does not contradict the non-redundancy guarantee for reasonable strategies, because learned clauses are guaranteed to be non-redundant only with respect to $G$.
Since \ICLF{} does not force all clauses to be instantiated, a clause that is redundant with respect to $N$ may still be non-redundant with respect to $G$.

\begin{restatable}[\ICLF{} Simulates Resolution]{remark}{remarkSimulationResolution} \label{remark:simulationResolution}
	Let $(\Theta;\Gamma;N;G;k;\top)$ be an \ICLF{} state reached in a reasonable run, and let $C_1\lor L_1,C_2\lor L_2\in N$ be such that $L_1$ and $\comp(L_2)$ are unifiable.
	Let $\mu=\mgu(L_1,\comp(L_2))$.
	If the resolvent of the clauses is not a tautology and neither parent clause becomes tautological under $\mu$, then \ICLF{} can learn it or a clause subsuming it by following a reasonable strategy as follows.
	We first apply Restart with $\Gamma_1=\eps$ to empty the trail, and then InstanceDel on all instances in $G$ to empty the set of ground instances.
	Next, we apply Instantiate with $(C_1\lor L_1)\cdot(\mu\sigma)$ and $(C_2\lor L_2)\cdot(\mu\sigma)$, where we pick $\sigma$ such that no two atoms $A_1,A_2\in\atom(C_1\cup C_2\cup\{L_1,L_2\})$ with $A_1\mu\neq A_2\mu$ become equal.
	This can, for example, be achieved by instantiating all variables with different fresh constants.
	Then, we decide the complements of all literals except for $L_1\mu\sigma$ and $L_2\mu\sigma$.
	This is possible because the resolvent of the clauses is not a tautology, because neither parent clause is tautological under $\mu$, and because of the way we chose $\sigma$.
	It also does not contradict the reasonable strategy because neither of the two ground instances can be false at this point.
	We then propagate $L_1\mu\sigma$ and apply Conflict with $(C_2\lor L_2)\cdot(\mu\sigma)$.
	Resolving exhaustively with $L_1\mu\sigma$ yields the desired resolvent, or, due to exhaustive factorization, a clause subsuming it.
	The derived clause is either $\bot$ or can be learned using BacktrackClassic.
\end{restatable}

Now, we show that \ICLF{} can simulate clause learning in SCL(T) \cite{DBLP:conf/vmcai/BrombergerFW21}.
Unlike in the simulation of SCL(FOL), we do not simulate every rule application.
Instead, we simulate the sequence of rule applications that leads to the application of Conflict and show that \ICLF{} can learn the same clause that SCL(T) learns.
We assume that a constrained clause $\Lambda\parallel C$ is just a different notation for the clause $\neg\Lambda\lor C$ and note the status closure for SCL(T) as $\overline D\cdot\overline\tau$, assuming that $\top=\top\cdot\{\}$.
Let $N'$ be a pure, abstracted clause set according to the requirements of SCL(T) and $B'$ be a finite sequence of constants of background sorts.

\begin{restatable}{definition}{definitionAssociatedSCLT} \label{definition:associatedSCLT}
	We say that an \ICLF{} state $(\Theta;\Gamma;N;G;k;D\cdot\tau)$ is \emph{weakly associated} with a given SCL(T) state $(\overline M;\overline N;\overline U;\overline B;\overline k;\overline D\cdot\overline\tau)$ if $N=\overline N\cup\overline U$, $D\cdot\tau=\overline D\cdot\overline\tau$, and the background literals in $\Gamma$ form a prefix of $\Gamma$.
	It is called \emph{associated} if additionally, the set of literals in $\overline M$ is a subset of the set of literals in $\Gamma$, the foreground literals in $\overline M$ form a subsequence of $\Gamma$, and for all propagations $L^{\overline C\cdot\overline\sigma}$ in $\overline M$, the corresponding literal $L$ in $\Gamma$ is annotated with a closure $C\cdot\sigma$ such that $\overline C\cdot\overline\sigma$ is the result of exhaustively factorizing $L$ in $C\cdot\sigma$, and for all decisions $L^k$ in $\overline M$, the corresponding literal $L$ in $\Gamma$ is also a decision.
\end{restatable}

\begin{restatable}[Simulation of Conflict]{theorem}{theoremSimulationSCLTConflict} \label{theorem:simulationSCLTConflict}
	For any regular SCL(T) run
	$$(\eps;N';\emptyset;B';0;\top)\shortruleapp{SCL(T)}{$\ast$}S_0\shortruleapp{SCL(T)}{$\ast$(no Confl)}S_1\shortruleapp{SCL(T)}{Conflict}S_2$$
	where $S_0$ is not a conflict state, and any state $T_0$ reached from $(\emptyset;\eps;N';\emptyset;0;\top)$ in a reasonable \ICLF{} run and weakly associated with $S_0$, there exists a sequence of reasonable \ICLF{} rule applications of the rules Instantiate, InstanceDel, \theory-Atom, Restart, Decide and Propagate from $T_0$ to a state $T_1$ that is associated with $S_1$, and Conflict can be applied to $T_1$ following a reasonable strategy, reaching a state $T_2$ that is associated with $S_2$.
\end{restatable}
\newcommand{\theoremSimulationSCLTConflictProof}{
\begin{proof}
	Since $S_0$ is not a conflict state and Conflict was not applied to get from $S_0$ to $S_1$, only Propagate, Decide and Grow can have been applied to get from $S_0$ to $S_1$.
	In particular, the sets of initial and learned clauses in $S_0$ and $S_1$ are the same.
	Let $\overline{M}$ be the trail of $S_1$.
	We begin by applying Restart to $T_0$ to empty the trail, which is clearly possible because $T_0$ is weakly associated with $S_0$, so it is not a conflict state.
	Now we apply Decide with all background literals in $\overline{M}$ in any order.
	If this is not possible because the corresponding atom is not contained in the current set of ground atoms, \theory-Atom is applied first.
	All added literals are propositionally undefined before their addition because we emptied the trail using Restart before and all literals in $\overline{M}$ are unique.
	Conflict also does not become applicable because this would imply the presence of a clause consisting only of background literals, which is equivalent to the empty clause in the context of SCL(T).
	Now, following the order of $\overline{M}$, we add the foreground literals in $\overline{M}$ using Propagate and Decide, according to their annotations in $\overline{M}$, using InstanceDel, Instantiate and \theory-Atom as needed, analogously to the proof of \Cref{theorem:simulationSCL}.
	For the propagations, we use the respective clause that was used in SCL(T).
	This is always possible because all background literals from $\overline{M}$ are on the trail.
	This does not violate the conditions of a reasonable run because in a regular SCL(T) run, Conflict has precedence over all other rules, Decide does not enable an application of the rule Conflict and \ICLF{} has the same clause set as SCL(T) in this situation.
	The fact that we added all background literals first cannot cause any conflicts because Conflict is applicable in SCL(T) if the foreground literals are propositionally false and the background literals are \theory-satisfiable.
	Now we potentially apply InstanceDel and Instantiate to enable the application of Conflict, analogously to the proof of \Cref{theorem:simulationSCL}.
	The state $T_1$ reached like this is associated with $S_1$.
	Since $T_0$ is weakly associated with $S_0$, the sets of clauses are equal.
	Clearly, the sets of literals on the trail are also equal, the foreground literals are in the same order, the propagations and decisions are annotated as claimed and the background literals form a prefix of the trail.
	After applying Conflict with the respective clause to $T_1$, we reach a state $T_2$ that is associated with $S_2$.
\end{proof}
}
\ifthenelse{\boolean{proofsInPaper}}{\theoremSimulationSCLTConflictProof}{}

\begin{restatable}[Simulation of Learning]{theorem}{theoremSimulationSCLTLearning} \label{theorem:simulationSCLTLearning}
	Consider a regular SCL(T) run
	$$(\eps;N';\emptyset;B';0;\top)\shortruleapp{SCL(T)}{$\ast$}S_0\shortruleapp{SCL(T)}{Conflict}S_1\shortruleapp{SCL(T)}{$\{\text{Resolve},\text{Factorize},\text{Skip}\}^\ast$}S_2$$
	and a state $T_1$ reached from $(\emptyset;\eps;N';\emptyset;0;\top)$ in a reasonable \ICLF{} run and associated with $S_1$.
	Then there is a state $T_2$ reachable from $T_1$ by applying the rules Resolve and Factorize in a reasonable way such that $T_2$ is associated with $S_2$.
	Moreover, if Backtrack can be applied to $S_2$ in a regular run, reaching a state $S_3$, then BacktrackClassic is applicable to $T_2$ following a reasonable strategy, reaching a state $T_3$ that is weakly associated with $S_3$.
\end{restatable}
\newcommand{\theoremSimulationSCLTLearningProof}{
\begin{proof}
	The proof of the first claim is by induction on the number of rules applied to reach $S_2$ from $S_1$.
	For the base case, $T_2:=T_1$ is obviously associated with $S_2$.
	For the induction step, let $S'$ be the state that $S_2$ was reached from and $T'$ be reachable from $T_1$ as claimed and associated with $S'$.
	If the rule applied to $S'$ to reach $S_2$ is Skip, $T_2:=T'$ is already associated with $S_2$.
	If the applied rule is Resolve, then Resolve can be applied to $T'$ to reach $T_2$ because $T'$ is associated with $S'$, so the propagation used by Resolve is also on the trail of $T'$.
	The conflict clauses in $S_2$ and $T_2$ are the same because SCL(T) exhaustively factorizes the annotated clause while propagating, whereas \ICLF{} exhaustively factorizes it while resolving, and the respective trails and clause sets do not change, so $T_2$ is associated with $S_2$.
	If the applied rule is Factorize, then Factorize can be applied to $T'$, reaching a state $T_2$ that is associated with $S_2$.

	Now assume that $T_2=(\Theta;\Gamma;N;G;k;D\cdot\tau)$ is associated with $S_2$ and Backtrack is applicable to $S_2$ in a regular SCL(T) run.
	Let $S_2=(\overline{M},\overline{K}^{i+1},\overline{M'};\overline N;\overline U;$ $\overline B;\overline k;(\overline\Lambda\parallel\overline D\lor\overline L)\cdot\overline\sigma)$.
	Further, let $\Gamma=\Gamma_1\Gamma_2\overline{K}^j\Gamma_3$ such that $\Gamma_1$ contains exactly all background literals of $\Gamma$.
	It is possible to split $\Gamma$ like this because $T_2$ is associated with $S_2$.
	Additionally, it holds that $D\cdot\tau=(\overline\Lambda\parallel\overline D\lor\overline L)\cdot\overline\sigma$.
	Since the foreground literals on the trail of $S_2$ are a subsequence of the literals in $\Gamma$, $\comp(\overline L\overline\sigma)$ must be contained in $\overline{K}^j\Gamma_3$ and the complements of the literals in $\overline D\overline\sigma$ must be contained in $\Gamma_2$.
	Obviously, the literals in $\overline\Lambda\overline\sigma$ must be contained in $\Gamma_1$.
	Hence, BacktrackClassic can be applied to $T_2$ reaching a state $T_3$ that is weakly associated with $S_3$ because \ICLF{} learns the same clause as SCL(T).
	Applying BacktrackClassic in this situation does not violate the reasonable strategy.
	For a contradiction, assume that the rightmost literal $L$ in $\Gamma$ occurs in $D\tau$.
	The conflict clause would be the empty clause in the context of SCL(T) if $L$ was a background literal because the background literals form a prefix of $\Gamma$, and since Backtrack is not applicable on the empty clause in SCL(T), $L$ must be a foreground literal.
	None of the rules applied to reach $T_2$ from $T_1$ changes the trail or introduces literals further to the right, so $L$ must have occurred in the conflict clause in $T_1$, hence also in $S_1$.
	Moreover, since the foreground literals on the trail of $S_1$ form a subsequence of the foreground literals on the trail of $T_1$ and the conflict clauses are the same, $L$ must occur in the conflict clause of $S_1$ and it must be the rightmost foreground literal on the trail of $S_1$.
	However, in a regular SCL(T) run, Resolve resolves away at least the rightmost foreground literal from the trail, contradiction.
\end{proof}
}
\ifthenelse{\boolean{proofsInPaper}}{\theoremSimulationSCLTLearningProof}{}

Together, \Cref{theorem:simulationSCLTConflict,theorem:simulationSCLTLearning} show how \ICLF{} can simulate a given regular SCL(T) derivation of the empty clause.
Clearly, the start state of \ICLF{} is weakly associated with the start state of SCL(T).
Whenever Conflict is applied in SCL(T), \ICLF{} can reach an associated conflict state by \Cref{theorem:simulationSCLTConflict}.
Then, all applications of Resolve, Factorize and Skip during conflict analysis in SCL(T) can be simulated in \ICLF{} by \Cref{theorem:simulationSCLTLearning}.
If SCL(T) derives the empty clause during conflict analysis, \ICLF{} also derives it because the states are shown to be associated.
If SCL(T) learns a clause, then by \Cref{theorem:simulationSCLTLearning}, \ICLF{} can learn the same clause, reaching a weakly associated state from which the described procedure can be iterated.

Next, we show that \ICLF{} can simulate the Isabelle/HOL verified calculus \CDCLW{} \cite{DBLP:journals/jar/BlanchetteFLW18}.
CDCL is an algorithm for pure propositional logic, so we assume \theory{} to be the free theory.
For easier notation, given a ground clause $D$, we assume that $D=D\cdot\{\}$.
To accommodate the different notations in \ICLF{} and \CDCLW{}, we assume that $L^j=L^\dagger$ for ground literals $L$, and that the trail in \CDCLW{} grows to the right instead of to the left.

\begin{restatable}{definition}{definitionAssociatedCDCL} \label{definition:associatedCDCL}
	We say that an \ICLF{} state $(\Theta;\Gamma;N;G;k;D\cdot\tau)$ is \emph{associated} with a given \CDCLW{} state $(\overline M;\overline N;\overline U;\overline D)$ if $\clauses(G)=\overline N\cup\overline U$, $D\cdot\tau=\overline D$, when $D\cdot\tau=\top$, then $\Gamma=\overline M$, and when $D\cdot\tau\neq\top$, then $\overline M$ is a prefix of $\Gamma$.
\end{restatable}

From its start state, \ICLF{} first has to instantiate all given clauses to reach a state that is associated with the \CDCLW{} start state.
Then, each rule application in \CDCLW{} can be simulated by at most one rule application in \ICLF{}.

\begin{restatable}[\ICLF{} Simulates \CDCLW{}]{theorem}{theoremSimulationCDCL} \label{theorem:simulationCDCL}
	Let $N'$ be a propositional clause set.
	Then for any reasonable \CDCLW{} run
	$$(\eps;N';\emptyset;\top)=S_0\shortruleapp{\CDCLW{}}{}S_1\shortruleapp{\CDCLW{}}{}\dots \shortruleapp{\textsc{\CDCLW{}}}{}S_m,$$
	there exists a reasonable \ICLF{} run
	$$(\emptyset;\eps;N';\emptyset;0;\top)\shortruleapp{\ICLF{}}{Instantiate $\ast$}T_0\shortruleapp{\ICLF{}}{$\leq1$}T_1\shortruleapp{\ICLF{}}{$\leq1$}\dots\shortruleapp{\ICLF{}}{$\leq1$}T_m$$
	such that for all $i\in\{0,\dots,m\}$, $T_i$ is associated with $S_i$.
	After the initial applications of Instantiate, only the \ICLF{} rules Propagate, Decide, Conflict, Restart, InstanceDel, Resolve and BacktrackClassic are needed.
\end{restatable}
\newcommand{\theoremSimulationCDCLProof}{
\begin{proof}
	To reach $T_0$ from the start state, Instantiate is applied with all clauses in $N'$ and the substitution $\{\}$.
	The reached state is clearly associated with $S_0$.
	Now the proof is by induction on $m$.
	Assume $T_{m'}$ is associated with $S_{m'}$ and a \CDCLW{} rule is applied to $S_{m'}$ to reach $S_{m'+1}$.
	By a case distinction on the applied rule, we show that we can apply at most one \ICLF{} rule to reach a state $T_{m'+1}$ that is associated with $S_{m'+1}$.
	Let $T_{m'}=(\Theta;\Gamma;N;G;k;D\cdot\tau)$ and $S_{m'}=(\overline M;\overline N;\overline U;\overline D)$.
	Propagate can be simulated by Propagate; the sets of ground clauses $\clauses(G)$ and $\overline N\cup\overline U$ are equal and the conditions of the two rules are the same.
	Decide can be simulated by Decide; the ground literal must be contained in $\Theta$ because it is contained in $\overline N$ and hence also in $\clauses(G)$, the remaining conditions of the two rules are the same.
	Conflict can be simulated by Conflict; the used clause is also contained in $\clauses(G)$, moreover it is unique by \Cref{corollary:gUnique}.
	Restart can be simulated by Restart with $\Gamma_1=\eps$.
	Forget can be simulated by InstanceDel.
	Again, by \Cref{corollary:gUnique}, the deleted clause occurs in $G$ at most once.
	If Skip was the applied rule, $T_{m'+1}:=T_{m'}$ is already associated with $S_{m'+1}$.
	Resolve can be simulated by Resolve.
	If the applied rule was Jump, then let $\overline M=MK^\dagger M'$ and $\overline D=D'\lor L$ such that $L$ has the level of the state $S_{m'}$.
	It follows that $D'$ is propositionally false in $M$ and $L$ is propositionally undefined in $M$, so BacktrackClassic is applicable to $T_{m'}$, restoring the equality of the trails and hence reaching a state $T_{m'+1}$ that is associated with $S_{m'+1}$.
	BacktrackClassic is indeed applicable following a reasonable strategy because in CDCL, due to exhaustive propagation, Resolve must be applied at least once before Jump becomes applicable.
\end{proof}
}
\ifthenelse{\boolean{proofsInPaper}}{\theoremSimulationCDCLProof}{}

In standard presentations of CDCL(T)~\cite{DBLP:journals/jacm/NieuwenhuisOT06,DBLP:series/faia/BarrettSST21}, conflict analysis is left flexible, allowing different learning mechanisms.
We therefore do not state a formal simulation theorem for CDCL(T).
Since \ICLF{} learns only non-redundant clauses, such a theorem would have to impose a compatible condition on learning in CDCL(T).
We can nevertheless describe how \ICLF{}, given a theory solver for \theory{}, can be used to decide the satisfiability of a ground clause set modulo \theory{}.

\begin{restatable}[\ICLF{} Simulates CDCL(T)]{remark}{remarkSimulationCDCLT}\label{remark:simulationCDCLT}
	Let $N'$ be a set of ground clauses modulo a theory \theory{}.
	We start in the state $(\emptyset;\eps;N';\emptyset;0;\top)$ and first apply Instantiate with $\{\}$ to all clauses in $N'$.
	Now, by \Cref{theorem:termination,theorem:stuckStates}, applying the rules Decide, Propagate, Conflict, Restart, \theory-Propagate, \theory-Learn, \theory-Atom, \theory-Conflict, Explain, Resolve and BacktrackClassic according to a reasonable strategy and restricting Restart, \theory-Learn and \theory-Atom to a finite number of applications, \ICLF{} either derives the empty clause, or it terminates in a state whose trail is a \theory{}-satisfiable model of $N'$.
\end{restatable}

For instantiation-based SMT procedures combining quantifier instantiation with CDCL(T), we are not aware of any formal description.
Nevertheless, \ICLF{} can model the essential interaction between an instantiation module and a CDCL(T) solver.
The process consists of two phases: the instantiation phase and the CDCL(T) phase.
First, the instantiation module generates some instances, including all of the initially ground clauses, using Instantiate.
Then, CDCL(T) is run according to \Cref{remark:simulationCDCLT}.
If it derives unsatisfiability, the process terminates.
Otherwise, the instantiation module adds more instances.
In \ICLF{}, due to non-ground learning, the instantiation module can also instantiate learned clauses, which is not possible in classical instantiation-based SMT solving.

There are other calculi that integrate explicit model building and a form of clause learning.
Theory instantiation~\cite{DBLP:conf/lpar/GanzingerK06} extends the InstGen~\cite{GanzingerKorovinEtAl03} calculus with theories. Still, only instances of existing clauses
are learned with respect to ground theory conflicts. The model-building in model evolution with lemma learning~\cite{DBLP:conf/lpar/BaumgartnerFT06} goes beyond ground
literals but does not involve theories.  Lemma learning gets more complicated and no non-redundancy guarantees are presented.
In conflict resolution~\cite{DBLP:journals/jar/SlaneyP18} models are built by first-order decision literals and propagations computed by the unit propagation rule.
The conflict resolution rule and the overall calculus are described in an abstract way, such that no quality guarantees are possible with respect to newly generated clauses.

\section{Conclusion} \label{section:conclusion}
We presented the calculus \ICLF{} for non-ground SMT, which generalizes several model-based approaches for propositional logic, first-order logic, and first-order logic modulo theories.
Most notably, \ICLF{} extends quantifier instantiation with non-ground learning and a tighter integration of instantiation and model building.
It also extends SCL by allowing a lazier treatment of non-ground clauses: conflicts only have to be considered for the current ground instances.
In this way, \ICLF{} provides a flexible setting for combining techniques from instantiation-based SMT solving and non-ground clause learning.
The long-term goal is to use this flexibility to develop new methods for solving problems that are beyond the reach of current solvers.

The main direction for future work is to implement \ICLF{} and integrate its ideas into existing solvers.
Since \ICLF{} generalizes multiple other calculi, it naturally leaves freedom in several choices.
One such choice is which ground instances to create and when to create them.
This difficulty already arises in quantifier instantiation and in SCL, and we hope that instantiation techniques that have proved useful in SMT solving~\cite{DBLP:journals/jacm/DetlefsNS05,DBLP:conf/cade/MouraB07,DBLP:conf/cav/GeM09,DBLP:conf/fmcad/ReynoldsTM14,DBLP:conf/tacas/ReynoldsBF18} can also be adapted to \ICLF{}.
Another such choice is when, and to what extent, to check the applicability of theory rules, for instance when to perform complete checks.
This depends on the specific theory, and we expect that strategies that work well in CDCL(T) will also work well in \ICLF{}.
As in all model-based calculi, there is also freedom in the choice of which literals to include in the model, that is, which decisions and propagations to perform.
Unlike in SCL, however, decisions and propagations are restricted by the current ground instances in \ICLF{}.
Finally, there is freedom in the choice of how to explain theory propagations and conflicts.
In CDCL(T), it is sufficient to explain a theory conflict by a minimal subset of contradictory trail literals, for example the literals $h(a)=g(a)$, $h(a)=a$, and $g(a)\neq a$ from \Cref{example:nonground}.
In \ICLF{}, this ground explanation can be used, but it may be preferable to use the more general explanation $x_2\neq x_3\lor x_2\neq x_4\lor x_3=x_4$, a theory lemma expressing the transitivity axiom for equality.
Instantiated with the substitution $\{x_2\mapsto h(a),x_3\mapsto g(a),x_4\mapsto a\}$, this yields the same ground explanation, but at the non-ground level it allows us to learn a more general clause from the conflict.
If we use the ground explanation directly, by contrast, then all variables matched against the ground explanation are grounded, and this restricts our ability to learn non-ground clauses.
Thus, besides instantiation heuristics, a central challenge for implementing \ICLF{} is to develop theory solvers that can produce useful non-ground explanations~\cite{DBLP:conf/cade/BarbosaRKLNNOPV22,leidinger2025computing}.

This gives another perspective on the role of theory explanations in \ICLF{}.
In saturation-based theorem proving, adding theory axioms as ordinary clauses can greatly enlarge the search space~\cite{DBLP:conf/lpar/RegerS17,DBLP:conf/cade/Gleiss020,DBLP:conf/tacas/KorovinKRSV23}.
In \ICLF{}, theory axioms need not be added to $N$ as clauses available for arbitrary instantiation, propagation, and resolution.
Instead, theory lemmas only have to be produced when they become relevant during conflict analysis.
In this way, the ground search guides the use of theory lemmas in resolution, rather than making all theory axioms available throughout the search.
This allows theory reasoning to contribute non-ground theory lemmas when the ground search needs them, combining the control of model-guided reasoning with the generality of first-order learning.

\paragraph{Acknowledgements}
We thank our anonymous reviewers for their constructive feedback.

\label{sect:bib}
\bibliographystyle{plain}
\bibliography{smt_vars}

\appendix

\section{Further Examples}
\begingroup
\allowdisplaybreaks
\begin{example} \label{example:redundantCB}
	This example shows that the 1UIP clause can be redundant in chronological backtracking.
	Consider the propositional clause set $N=\{P,\neg P\lor\neg R\lor S,\neg V\lor U,\neg V\lor\neg U,T\lor V\lor U,T\lor V\lor\neg U,\neg P\lor\neg T\lor V,\neg R\lor\neg S\lor\neg T\lor V\}$.
	We use the calculus by Möhle and Biere \cite{DBLP:conf/sat/MohleB19}, but omit $\delta$ and annotate the levels on the trail instead.
	This means that our states are two-tuples $(N;M)$ where $N$ is the clause set and $M$ is the trail.
	We always use the 1UIP clause for backtracking and always perform chronological backtracking, i.e., we always jump to the previous decision level.
	\begin{align*}
		& &&(N;\eps)&&&\\
		&\leadsto_\text{Unit}&&(N;P^P)&&&\\
		&\leadsto_\text{Decide}&&(N;P^PR^1)&&&\\
		&\leadsto_\text{Unit}&&(N;P^PR^1S^{\neg P\lor\neg R\lor S})&&&\\
		&\leadsto_\text{Decide}&&(N;P^PR^1S^{\neg P\lor\neg R\lor S}\neg T^2)&&&\\
		&\leadsto_\text{Decide}&&(N;P^PR^1S^{\neg P\lor\neg R\lor S}\neg T^2V^3)&&&\\
		&\leadsto_\text{Unit}&&(N;P^PR^1S^{\neg P\lor\neg R\lor S}\neg T^2V^3U^{\neg V\lor U})
		\intertext{Now, the clause $\neg V\lor\neg U$ is false, both literals in this clause are of decision level $3$. The 1UIP clause is $\neg V$, obtained by resolving with the clause $\neg V\lor U$ from which $U$ was propagated.}
		&\leadsto_\text{Jump}&&(N':=N\cup\{\neg V\};P^PR^1S^{\neg P\lor\neg R\lor S}\neg T^2\neg V^{\neg V})&&&\\
		&\leadsto_\text{Unit}&&(N';P^PR^1S^{\neg P\lor\neg R\lor S}\neg T^2\neg V^{\neg V}U^{T\lor V\lor U})&&&\\
		\intertext{Now, the clause $T\lor V\lor\neg U$ is false. $U$ and $T$ are of decision level $2$, but $V$ is of decision level $0$, so the 1UIP clause is $T\lor V$, obtained by resolving with the clause $T\lor V\lor U$ from which $U$ was propagated.}
		&\leadsto_\text{Jump}&&(N'':=N'\cup\{T\lor V\};P^PR^1S^{\neg P\lor\neg R\lor S}\neg V^{\neg V}T^{V\lor T})
	\end{align*}
	Now, the clause $\neg R\lor\neg S\lor\neg T\lor V$ is false.
	$R$ and $S$ are of decision level $1$, $T$ and $V$ are of decision level $0$.
	Hence, the 1UIP clause is $\neg P\lor\neg R\lor\neg T\lor V$, obtained by resolving with the clause $\neg P\lor\neg R\lor S$ from which $S$ was propagated.
	This clause is redundant with respect to any ordering due to the clause $\neg P\lor\neg T\lor V\in N\subseteq N''$.
\end{example}
\endgroup

\section{Further Lemmas}
\begin{restatable}[Stuck SCL(FOL) States Are Stuck in \ICLF{}]{lemma}{lemmaSCLStuckSimulation}\label{lemma:SCLStuckSimulation}
	If $$(\eps;N';\emptyset;\beta';0;\top)\shortruleapp{SCL(FOL)}{$\ast$}(\overline\Gamma;\overline N;\overline U;\overline\beta;\overline k;\overline D\cdot\overline\tau)=:S$$ in a regular SCL(FOL) run, no rule except Grow is applicable to $S$, $$(\emptyset;\eps;N';\emptyset;0;\top)\shortruleapp{\ICLF{}}{$\ast$}(\Theta;\Gamma;N;G;k;D\cdot\tau)=:T$$ in a reasonable \ICLF{} run and $T$ is associated with $S$, then no \ICLF{} rule except Instantiate, ClauseDel, InstanceDel, AtomDel, Restart, \theory-Learn and \theory-Atom is applicable to $T$.

	Moreover, there is no atom $A\prec_B\overline\beta$ such that $A$ is a ground instance of an atom in $\atom(N)$ and \theory-Atom is applicable with $A$.
\end{restatable}
\newcommand{\lemmaSCLStuckSimulationProof}{
\begin{proof}
	The rule Explain is not applicable in the free theory because no undefined literals can be theory-propagated from a consistent trail in the free theory.
	If no rules except Grow are applicable to $(\overline\Gamma;\overline N;\overline U;\overline\beta;\overline k;\overline D\cdot\overline\tau)$, then either $\overline D\cdot\overline\tau=\bot$ and $\overline N$ is unsatisfiable, or $\overline D\cdot\overline\tau=\top$, $\gnd(N)^{\prec_B\overline\beta}$ is satisfiable and $\overline\Gamma\models\gnd(N)^{\prec_B\overline\beta}$ \cite[Theorem 10]{DBLP:journals/corr/abs-2302-05954}.

	If $\overline D\cdot\overline\tau=\bot$, then $D\cdot\tau=\bot$ as well because $T$ is associated with $S$.
	Clearly, no rule except Explain can be applied to $T$ in this case, and as argued earlier, Explain is not applicable in the free theory.

	Otherwise, $D\cdot\tau=\overline D\cdot\overline\tau=\top$ and $\overline\Gamma\models\gnd(N)^{\prec_B\overline\beta}$.
	We claim that there is no atom $A\in\Theta$ such that $A$ is propositionally undefined on the trail.
	If there was such an atom $A$, then $A\prec_B\overline\beta$ and there would be a clause $C\in N=\overline N\cup\overline U$ such that $A$ is a ground instance of a literal in $C$.
	It follows that Decide would be applicable to $S$, contradiction.
	For each rule that we claim that it is not applicable to $T$, we argue why.
	For Resolve, Factorize, BacktrackClassic, BacktrackFOL and BacktrackCB, it is obvious that they are not applicable to $T$ because $D\cdot\tau=\top$.
	For Decide, Propagate and \theory-Propagate, this follows from the fact that there is no propositionally undefined atom in $\Theta$.
	If Conflict was applicable to $T$, that is, if there was a $C\cdot\sigma\in G$ such that $C\sigma$ is propositionally false in $\Gamma$, then by \Cref{property:Ginstances} of \Cref{lemma:invariants}, $C\in\overline N\cup\overline U$, so Conflict would be applicable to $S$, contradiction.
	If \theory-Atom was applicable with an atom $A\prec_B\overline\beta$ that is a ground instance of an atom in $\atom(N)$, then Decide would be applicable to $S$ because $A$ is undefined in $\Gamma$ and $\overline\Gamma$ by \Cref{property:thetaContains} of \Cref{lemma:invariants}, contradiction.
	Since \theory{} is the free theory and $\Gamma$ is consistent, no theory-valid clause can be propositionally false in $\Gamma$, so \theory-Conflict is not applicable to $T$.
\end{proof}
}
\lemmaSCLStuckSimulationProof

\section{A Strategy to Apply BacktrackCB}
In the conclusion of their paper, Möhle and Biere \cite{DBLP:conf/sat/MohleB19} suspect that chronological backtracking can also be applied in SMT.
Following this, we incorporate a rule BacktrackCB that allows to backtrack in a way that is similar to chronological backtracking.
However, in the context of SMT, a few challenges arise that do not occur in pure SAT solving.
Most importantly, with theory propagations, the actual decision level, called asserting level by Möhle and Biere, of propagated literals is not always known.
This affects not only the theory propagated literals, but also the literals propagated from these literals.
As explaining theory propagations can be expensive, we aim for a conflict analysis procedure that needs to explain as little as possible.
This results in three key changes in comparison to the procedure described by Möhle and Biere.
First, we do not define our procedure based on decision levels, but based on reachable literals from the rightmost implying decision $L'^{j+1}$ instead.
This means that no literals to the left of this decision $L'^{j+1}$ need to be explained.
Second, instead of keeping all literals implied by decisions to the left of $L'^{j+1}$, we only keep, in addition to all literals to the left of $L'^{j+1}$, those literals to the right of $L'^{j+1}$ that are in $\reach^{-1}$ of the learned clause, but not implied by $L'$.
This is because these literals are necessary to propagate from the learned clause, and for any other literal to the right of $L'^{j+1}$, we would first have to find its actual decision level, possibly causing unnecessary explaining.
Third, whereas their procedure allows backtracking to any decision level between the one at which the learned clause propagates and $j$, we only allow backtracking to decision level $j$, resulting in a much simpler definition.
Otherwise, all theory propagations between $L'^{j+1}$ and the point to which we backtrack would have to be explained.

In the following, we describe how the rules Explain, Resolve and BacktrackCB can be used to perform conflict analysis on a ground conflict clause $D'$ with chronological backtracking.
Let $D'=D''\lor L''$ such that $\comp(L'')$ occurs rightmost on the trail among the complements of literals in $D'$.
If $\comp(L'')$ is a decision, then by \Cref{lemma:conflictOnPropagation}, the conflict was reached by \theory-Conflict and BacktrackCB can be applied immediately.
Otherwise, Resolve should be applied to $\comp(L'')$, which possibly requires explaining it first.
This ensures learning a non-redundant clause, and is required in a reasonable strategy [\Cref{definition:reasonable}].
We call the resulting clause $D$.
Then, using \Cref{algorithm:prepareCB} on $D$, we identify the rightmost implying decision $L'^{j+1}$ and mark all literals in $\reach(L')\cap\bigcup_{L\in D}\reach^{-1}(\comp(L))$.
Now, we can work backwards along the trail until we reach $L'^{j+1}$ to compute the learned clause.
When we encounter a literal whose complement occurs in the current conflict clause, then if it is not marked, we push it to a stack.
We also push all predecessors of literals on the stack to the stack, which can be managed by a second mark.
If the literal is the only marked trail literal whose complement occurs in the current conflict clause, then we skip it, as its complement is the literal that will be propagated from the learned clause.
If the literal is one of several marked trail literals whose complements occur in the current conflict clause, then we apply Resolve on it exhaustively, or we apply Factorize exhaustively and then Resolve.
Note that it is not necessary anymore to use Explain because all literals that are resolved now have been explained in \Cref{algorithm:prepareCB}.
In all other cases, i.e., if the complement of the literal does not occur in the current conflict clause and is also not the predecessor of any literal on the stack, we can safely skip this literal.
In the state that is reached by this procedure, BacktrackCB is applicable.

\begin{breakablealgorithm}
	\caption{Preparation for conflict analysis with BacktrackCB}\label{algorithm:prepareCB}
	\begin{algorithmic}[1]
		\Function{PrepareConflictAnalysisCB}{conflict clause $D$}
			\State $\textit{children}\gets$ data structure to store for each literal a list of literals
			\State $\textit{pq}\gets$ priority queue ordered by trail position (largest first)
			\For{$L\in D$}
				\State $\textit{pq}.\text{insert}(\comp(L))$
			\EndFor
			\While{$!\textit{pq}.\text{empty}()$ and $\textit{pq}.\text{front}()$ is not a decision}
				\State $L\gets\textit{pq}.\text{pop}()$
				\If{$L$ is a theory propagation}
					\State $\text{explain}(L)$
				\EndIf
				\For{$K\in\text{parents}(L)$}
					\State $\textit{children}[K].\text{insert}(L)$
					\If{$K$ not in $\textit{pq}$}
						\State $\textit{pq}.\text{insert}(K)$
					\EndIf
				\EndFor
			\EndWhile
			\If{$\textit{pq}.\text{empty}()$}
				\State \Return UNSAT
			\EndIf
			\State DFS from $\textit{pq}.\text{front}()$ using \textit{children}, mark all visited literals
		\EndFunction
	\end{algorithmic}
\end{breakablealgorithm}

\end{document}